\documentclass[times,final]{elsarticle}
\usepackage{mathtools}

\usepackage{amssymb,amsmath,amsthm}
\usepackage{amsfonts}
\usepackage{mathrsfs}
\usepackage{graphicx,subfigure,epstopdf}
\graphicspath{{./figures_cv/}}
\usepackage{latexsym,float,tikz}
\usepackage[ruled,vlined,linesnumbered]{algorithm2e}
\usepackage{booktabs,tabularx}
\newcolumntype{L}{>{\raggedright\arraybackslash}X}
\usepackage{siunitx}
\usepackage{makecell}
\usepackage{threeparttable}
\RestyleAlgo{ruled}
\usepackage{adjustbox}
\usepackage{multirow}

\usepackage{verbatim}
\usepackage{enumerate}
\usepackage{pdflscape}
\usepackage{enumitem}
\usepackage{bbm}
\numberwithin{equation}{section}
\newtheorem{theorem}{Theorem}[section]
\newtheorem{lemma}{Lemma}[section]
\newtheorem{definition}{Definition}[section]
\newtheorem{proposition}{Proposition}[section]

\newtheorem{remark}[theorem]{Remark}
\newtheorem{assumption}{Assumption}[section]

\usepackage{xcolor}
\usepackage[linktocpage]{hyperref}

\usepackage{jcomp}
\usepackage{framed,multirow}

\usepackage{amssymb}
\usepackage{latexsym}

\usepackage{url}
\usepackage{xcolor}
\definecolor{newcolor}{rgb}{.8,.349,.1}

\journal{Journal of Computational Physics}

\begin{document}

\verso{Yueqi Wang \textit{etal}}

\begin{frontmatter}

\title{Reduced--Basis Deep Operator Learning for Parametric PDEs with Independently Varying Boundary and Source Data}%

\author[1]{Yueqi Wang}
\author[1,2]{Guang Lin\corref{cor1}}
\cortext[cor1]{Corresponding author: Guang Lin, E-Mail: Guanglin@purdue.edu}

\address[1]{Department of Mathematics, Purdue University, 610 Purdue Mall, West Lafayette, 47907, IN, USA}
\address[2]{School of Mechanical Engineering, Purdue University, 610 Purdue Mall, West Lafayette, 47907, IN, USA}


\begin{abstract}
Parametric PDEs power modern simulation, design, and digital-twin systems, yet their many-query workloads still hinge on repeatedly solving large finite-element systems. Existing operator-learning approaches accelerate this process but often rely on opaque learned trunks, require extensive labeled data, or break down when boundary and source data vary independently from physical parameters. We introduce RB--DeepONet, a hybrid operator-learning framework that fuses reduced-basis (RB) numerical structure with the branch–trunk architecture of DeepONet. The trunk is fixed to a rigorously constructed RB space generated offline via Greedy selection, granting physical interpretability, stability, and certified error control. The branch network predicts only RB coefficients and is trained label-free using a projected variational residual that targets the RB--Galerkin solution. For problems with independently varying loads or boundary conditions, we develop boundary and source modal encodings that compress exogenous data into low-dimensional coordinates while preserving accuracy. Combined with affine or  empirical interpolation decompositions, RB--DeepONet achieves a strict offline–online split: all heavy lifting occurs offline, and online evaluation scales only with the RB dimension rather than the full mesh. We provide convergence guarantees separating RB approximation error from statistical learning error, and numerical experiments show that RB--DeepONet attains accuracy competitive with intrusive RB--Galerkin, POD--DeepONet, and FEONet while using dramatically fewer trainable parameters and achieving significant speedups. This establishes RB--DeepONet as an efficient, stable, and interpretable operator learner for large-scale parametric PDEs.
\end{abstract}


\end{frontmatter}


\section{Introduction}

Computational prediction for parametric PDEs underpins design, control, and digital-twin applications in science and engineering \cite{benner2015survey,hesthaven2016certified,kapteyn2022data}. Classical discretizations, such as finite difference \cite{strikwerda2004finite}, finite volume \cite{eymard2000finite}, and especially finite element methods (FEM) \cite{ciarlet2002finite}, provide robust high-fidelity solvers, but their cost scales with the full spatial degrees of freedom and must be paid anew for each parameter query. 
Consequently, tasks that require repeated forward solves, such as parametric sweeps \cite{hesthaven2016certified}, PDE-constrained optimization \cite{hinze2008optimization}, and uncertainty quantification \cite{sullivan2015introduction}, can become computationally prohibitive on fine meshes or in high-dimensional parameter spaces, even though the finite element method remains the predominant high-fidelity discretization in engineering analysis.
 
Projection-based model reduction techniques \cite{lucia2004reduced,maday2006reduced} alleviate this bottleneck by constructing low-dimensional subspaces that approximate the solution manifold over the parameter domain. Among them, Reduced Basis (RB) method \cite{quarteroni2015reduced} offers a rigorously certified offline–online framework: The offline stage generates high-fidelity snapshots and constructs a reduced trial space, typically via Proper Orthogonal Decomposition (POD) \cite{hesthaven2016certified,quarteroni2015reduced,liang2002proper} or a Greedy selection \cite{billaud2017dynamical,lappano2016greedy,hesthaven2014efficient}. 
The online stage performs inexpensive Galerkin projection \cite{rowley2004model,wang2020recurrent} onto this reduced space.
Under an affine or empirical interpolation decomposition, all parameter-independent operators can be preassembled, so the online complexity depends only on the reduced dimension $N$ rather than the full-order dimension $N_h$. Rigorous \emph{a posteriori} error estimators accompany each query, providing certified accuracy \cite{hesthaven2016certified,quarteroni2015reduced,rozza2008reduced}.
Nevertheless, classical RB methods remain intrusive and are not easily adaptable to heterogeneous or independently varying boundary and source data \cite{mcquarrie2023nonintrusive,gunzburger2007reduced,cosimo2016general}.
Nevertheless, classical RB methods remain intrusive and are not easily adaptable to heterogeneous or independently varying boundary and source data \cite{gunzburger2007reduced,cosimo2016general,mcquarrie2023nonintrusive}.

In parallel, machine learning has shown considerable promise for accelerating the numerical solutions of PDEs. Physics-Informed Neural Networks (PINNs) \cite{raissi2019physics} incorporate PDE residuals and boundary conditions directly into the loss function, enabling solution approximation without labeled data. Building on this idea, a large body of work has proposed variants of PINNs to tackle a wide range of PDE and physics-based problems \cite{karniadakis2021physics,lu2021deepxde,meng2020ppinn,goswami2022physics,patel2021physics}. However, standard PINNs and their extensions are typically designed for a single boundary or forcing configuration and must be retrained when the input data changes. In addition, the trial space is only defined implicitly by the network, which hinders the exact enforcement of boundary conditions. These limitations make PINN-type methods difficult to apply to parametric PDEs or to problems where the input data vary across queries.

Operator-learning frameworks overcome these limitations by learning mappings between function spaces
and are therefore well-suited to parametric PDEs.
Fourier Neural Operators (FNOs) \cite{li2020fourier} parameterize the operator via Fourier convolutions, injecting a spectral inductive bias that is particularly effective on regular grids.
DeepONet \cite{lu2019deeponet} approximates nonlinear operators via a branch–trunk architecture and is supported by the universal approximation theorem for operators \cite{chen1995universal}.
Physics-Informed Neural Operators (PINO) \cite{li2024physics} further combine operator learning with physics-informed residual losses to improve data efficiency.
Most recently, hybrid operator learning methods integrate classical numerical structure into the trunk network of DeepONet, leading to several variants that differ primarily in the choice of basis.
The Laplacian Eigenfunction-based Neural Operator (LE-NO)  \cite{hao2025laplacian} 
uses Laplacian eigenfunctions as the trunk,
enabling efficient approximation on nonlinear parabolic PDEs.
POD--DeepONet \cite{lu2022comprehensive} projects solutions onto a POD basis and learns only the corresponding coefficients, yielding a low-rank representation in a reduced output space.
Finite Element Operator Network (FEONet) \cite{lee2025finite} fixes the trunk to finite element basis and enforces the FEM variational residual directly.
However, these methods typically rely on learned trunks with limited physical
interpretability, require large labeled datasets, or predict full-order fields
with dimension $N_h$---hindering online efficiency.
To handle parametric PDEs more efficiently, a natural idea is therefore to fix the trunk to a reduced basis that captures the solution manifold with as few modes as possible, thereby combining model reduction with operator learning.

Actually, there is a growing line of work that couples model reduction techniques with learning for parametric PDEs.
Hesthaven \& Ubbiali \cite{hesthaven2018non} introduced a non-intrusive POD–ANN pipeline that learns a regression map from parameters to reduced coefficients.
In \cite{dal2020data}, the PDE-aware deep neural network (PDE-DNN) uses a reduced basis solver as the activation function in the last layer, so that the network outputs parameters for a precomputed RB model.
Chen et al. proposed physics-reinforced neural networks (PRNN) \cite{chen2021physics}, which map parameters to RB coefficients using a loss that blends the RB--Galerkin residual with coefficient labels.
Sentz et al. \cite{sentz2021reduced} trained a neural timestepping map that propagates RB coefficients for dynamical PDEs.
O’Leary-Roseberry et al. proposed a reduced basis adaptive residual network \cite{o2022learning} for learning high-dimensional parametric maps.
For further combinations of reduced basis methods with neural networks, we refer to \cite{hesthaven2018non,wang2019non,lee2020model}.
More recently, several works have combined model reduction with operator learning.
In \cite{mcquarrie2021data} and \cite{kramer2024learning}, reduced operators are identified directly from data via operator inference.
RO-NORM \cite{meng2024general} projects spatio-temporal fields onto a POD basis and then learns the mapping between these reduced functions via the NORM neural operator on Riemannian manifolds.
Reduced Basis Neural Operator (ReBaNO) \cite{zheng2025rebano} constructs a reduced basis from high-fidelity PINN solutions selected via a greedy procedure, and represents the solutions of parametric PDEs in this low-dimensional RB space. 

Nevertheless, to the best of our knowledge, existing approaches do not provide a clear framework that operates entirely in a reduced basis space while simultaneously handling parametric dependence and independently varying external data (e.g., source terms and boundary conditions). To address this gap, we propose in this paper the Reduced-Basis DeepONet (RB--DeepONet), a hybrid operator-learning framework that combines the DeepONet architecture with Reduced Basis method.

Our key contributions are:
\begin{itemize}
    \item \textbf{RB trunk}: We fix the DeepONet trunk to a rigorously constructed RB
    basis obtained via Greedy selection,
    yielding interpretability, stability, and certified error control.
    \item \textbf{Label-free residual training}: The branch network predicts only RB
    coefficients, trained by minimizing the RB-projected variational residual.
    This guarantees convergence to RB--Galerkin solutions on the training
    distribution without paired solution labels.
    \item \textbf{Independently varying data}: To handle boundary and source inputs
    that vary independently of the physical parameters, we introduce compact
    boundary and source modal encodings, producing a low-dimensional augmented
    input vector.
    \item \textbf{Strict offline--online split}: Under affine decomposition or empirical
    interpolation (EIM), all parameter-independent reduced operators are
    preassembled offline, making the online cost depend only on the RB dimension
    $N \ll N_h$.
    \item \textbf{Convergence theory}: We establish bounds that separate RB
    approximation error from statistical learning error, ensuring reliable
    generalization.
\end{itemize}
Numerical experiments show that RB--DeepONet achieves accuracy comparable to
intrusive RB--Galerkin, POD--DeepONet \cite{lu2019deeponet}, and FEONet \cite{lee2025finite},
while requiring significantly fewer trainable parameters and providing substantial
online speedups.

The remainder of the paper is organized as follows. Section \ref{sec:Problem Setting} reviews the variational formulation, finite element discretization, and the lifted setting we adopt.
Section \ref{sec:regimes} first develops RB--DeepONet for fully parameterized PDEs, including the construction of the RB trunk and the residual-minimization training objective. It then extends the framework to operator-parameterized problems with exogenous data by introducing boundary and source modes together with lifting and trace integration.
Section \ref{sec:Convergence analysis of RB--DeepONet} presents the convergence and generalization analysis. 
Section \ref{sec:Numerical Results} reports numerical experiments, 
and Section \ref{sec:Conclusions and future work} concludes with a discussion and perspectives.

\section{Problem Setting}\label{sec:Problem Setting}

For clarity of exposition, we develop the method in the prototypical second-order elliptic setting \eqref{eq:strong}. The proposed algorithm, however, extends to more general operators with only minor modifications, as summarized in Remark~\ref{remark:Generality}.

\subsection{Parametric PDEs}
Let $\Omega\subset\mathbb R^d$, $d\in\mathbb{N},$ be a bounded Lipschitz domain with boundary
$\partial\Omega:=\Gamma_D\cup\Gamma_N\cup\Gamma_R$, where the partition is pairwise disjoint corresponding to Dirichlet, Neumann and Robin parts, respectively.
Let $\mathcal D\subset\mathbb R^p$, $p\geq 1$, be a compact parameter set and
$\mathbf k\in\mathcal D$ a parameter vector.
For a given parameter $\mathbf{k}\in\mathcal{D}$, we consider the family of elliptic boundary value problems
\begin{equation}\label{eq:strong}
\begin{cases}
-\nabla\!\cdot\!\big(A(\mathbf x;\mathbf k)\nabla u(\mathbf x;\mathbf k)\big)
+ c_0(\mathbf x;\mathbf k)\,u(\mathbf x;\mathbf k) = f(\mathbf x;\mathbf k),
& \mathbf x\in\Omega,\\
\gamma_D u(\mathbf x;\mathbf k) = g_D(\mathbf x;\mathbf k), & \mathbf x\in\Gamma_D,\\
\big(A(\mathbf x;\mathbf k)\nabla u(\mathbf x;\mathbf k)\big)\!\cdot\!\boldsymbol n
= h_N(\mathbf x;\mathbf k), & \mathbf x\in\Gamma_N,\\
\big(A(\mathbf x;\mathbf k)\nabla u(\mathbf x;\mathbf k)\big)\!\cdot\!\boldsymbol n
+ \beta(\mathbf x;\mathbf k)\,u(\mathbf x;\mathbf k) = r_R(\mathbf x;\mathbf k),
& \mathbf x\in\Gamma_R,
\end{cases}
\end{equation}
where $\gamma_D:V:=H^1(\Omega)\to G_D:=H^{1/2}(\Gamma_D)$ denotes the Dirichlet trace on $\Gamma_D$ and
$\boldsymbol n$ the unit outward normal on $\partial\Omega$.
For each $\mathbf k\in\mathcal D$, we assume $A(\cdot;\mathbf{k}), c_0(\cdot;\mathbf{k}),\beta(\cdot;\mathbf{k})$ are in $L^{\infty}$ spaces and the data are taken in the natural spaces
\[
  f(\cdot;\mathbf k)\in L^2(\Omega),\quad
  g_D(\cdot;\mathbf k)\in H^{1/2}(\Gamma_D),\quad
  h_N(\cdot;\mathbf k)\in H^{-1/2}(\Gamma_N),\quad
  r_R(\cdot;\mathbf k)\in H^{-1/2}(\Gamma_R),
\]
which are uniformly bounded w.r.t. $\mathbf{k}$.
The weak form of \eqref{eq:strong} reads: for given $\mathbf k\in\mathcal D$, find
$u(\mathbf k)\in V$ with $\gamma_D u(\mathbf k)=g_D(\mathbf k)$ such that
\begin{equation}\label{eq:weak}
a\big(u(\mathbf k),v;\mathbf k\big)=\ell(v;\mathbf k)
\qquad \forall v\in V_0.
\end{equation}
For brevity, we henceforth omit the explicit spatial variable and write $u(\mathbf{x};\mathbf{k}),g(\mathbf{x};\mathbf{k})$ for $u(\mathbf{k}),g(\mathbf{k})$.
We set $V_0:=\{v\in H^1(\Omega):\, \gamma_D v=0\}$ and define the bilinear form
\[
a(u,v;\mathbf k):=\int_\Omega A\nabla u\!\cdot\!\nabla v
                 +\int_\Omega c_0\,uv
                 +\int_{\Gamma_R}\beta\,\gamma_R u\,\gamma_R v,
\]
and linear form
\[
\ell(v;\mathbf{k}):=\int_\Omega f\,v+\int_{\Gamma_N} h_N\,\gamma_N v+\int_{\Gamma_R} r_R\,\gamma_R v.
\]
Here, $\gamma_E: H^1(\Omega)\to H^{1/2}(\Gamma_E)$ is the trace operator restricted on $\Gamma_E$ for $E\in\{N,R\}$.
For each $\mathbf k\in\mathcal D$, we fix a reference parameter $\mathbf k_\star\in\mathcal D$ and use
\begin{equation*}
(u,v)_V:=a\big(u,v;\mathbf k_\star\big)
\end{equation*}
as the reference energy inner product in $V$ and $V_0$, with associated norm $\|u\|_V:=\sqrt{(u,u)_V}$.

We impose the following assumptions in what follows.

\begin{assumption}\label{ass:wellposed}
We assume that $a:V\times V\to\mathbb{R}$ is a continuous and symmetric bilinear form
and that $\ell:V\to\mathbb{R}$ is a linear functional such that there exist
constants $M,\alpha_{\mathrm{LB}},C_\ell>0$ independent of $\mathbf k\in\mathcal D$ with
\begin{align}
  |a(u,v;\mathbf k)| &\le M\,\|u\|_V\|v\|_V, &&\forall u,v\in V, \label{eq:cont}\\
  a(v,v;\mathbf k) &\ge \alpha_{\mathrm{LB}}\,\|v\|_V^2, &&\forall v\in V_0, \label{eq:coerc}\\
  |\ell(v;\mathbf{k})| &\le C_\ell\|v\|_V,&&\forall v\in V_0. \label{eq:rhs}
\end{align}
In addition, we assume that the Dirichlet trace $\gamma_D:V\to G_D$ admits a bounded
right inverse $T:G_D\to V$, i.e.,
\begin{equation}\label{eq:T}
     \gamma_D(Tg)=g,
  \quad \|Tg\|_V \le C_{\mathrm{tr}}\|g\|_{G_D},
  \qquad\forall g\in G_D,
\end{equation}
for some constant $C_{\mathrm{tr}}>0$ independent of $\mathbf k$.
\end{assumption}

Under Assumption~\ref{ass:wellposed}, the Lax--Milgram theorem \cite{evans2022partial} implies that, for each $\mathbf k\in\mathcal D$,
there exists a unique $u(\mathbf{k})\in V$ with $\gamma_D u(\mathbf k)=g_D(\mathbf k)$
solving the weak problem~\eqref{eq:weak}.
Moreover,
there exists a unique bounded linear lifting operator
$E_D:G_D\to V$ such that
\begin{equation}\label{eq:E_D}
      \gamma_D(E_D g)=g, \qquad
  a(E_D g, v;\mathbf k_\star)=0 \quad \forall v\in V_0 .
\end{equation}
This operator induces an inner product and norm on $G_D$ via
\begin{equation}\label{eq: boundary norm}
    \langle g_1,g_2\rangle_{D,\star}
  := a(E_D g_1, E_D g_2;\mathbf k_\star),
  \qquad
  \|g\|_{D,\star}:=\langle g,g\rangle_{D,\star}^{1/2}.
\end{equation}
With respect to $(\cdot,\cdot)_V$, $E_D$ is an isometry from
$(G_D,\langle\cdot,\cdot\rangle_{D,\star})$ onto its range in $(V,(\cdot,\cdot)_V)$,
i.e., $\|E_D g\|_V=\|g\|_{D,\star}$ for all $g\in G_D$.
Furthermore, for any finite-dimensional subspace $E_r\subset G_D$, if
$P_{E_r}:G_D\to E_r$ denotes the $\langle\cdot,\cdot\rangle_{D,\star}$-orthogonal
projection and $P_{E_D(E_r)}:V\to E_D(E_r)$ the $(\cdot,\cdot)_V$-orthogonal
projection, then
\[
  E_D\big(P_{E_r} g\big) \;=\; P_{E_D(E_r)}\big(E_D g\big),
  \qquad \forall g\in G_D .
\]

Writing the solution of \eqref{eq:strong} as $u(\mathbf k)=w(\mathbf k)+E_D[g_D(\mathbf k)]$
with $w(\mathbf k)\in V_0$,
the weak problem is equivalent to: for given $\mathbf k\in\mathcal D$, find
$w(\mathbf k)\in V_0$ such that
\begin{equation}\label{eq:ps:hom}
  a\big(w(\mathbf k),v;\mathbf k\big)=\mathfrak F(\mathbf k)[v]
  \qquad \forall v\in V_0 ,
\end{equation}
where the aggregated load functional is
\begin{equation}\label{eq:ps:aggload}
  \mathfrak F(\mathbf k)[v]
  := \ell(v;\mathbf k)-a\!\big(E_D[g_D(\mathbf k)],v;\mathbf k\big),
  \qquad \forall v\in V_0 .
\end{equation}
By Assumption \ref{ass:wellposed} and the construction of $E_D$ as in \eqref{eq:E_D}, there exists $C_F>0$ such that
\begin{equation*}
    |\mathfrak F(\mathbf k)[v]| \le C_F\|v\|_{V},
\end{equation*}
for all $\mathbf{k}\in\mathcal{D}$ and $v\in V_0$.
In the sequel, we work with \eqref{eq:ps:hom}.

\subsection{Finite Element approximation}\label{sec:Finite Element Approximation}

In the theory of FEM, let $\mathcal{T}_h$ be a shape-regular FE mesh and $V_h \subset V$ a conforming $P_1$ space.
We define the finite-dimensional ansatz spaces
\(V_h := S_h \cap V\) and \(V_{h,0} := S_h \cap V_0\).
Let \(\{\phi_i\}_{i=1}^{N_h}\) be the standard nodal basis of \(V_h\)
associated with the mesh nodes \(\{\mathbf x_i\}_{i=1}^{N_h}\).
We split the index set of degrees of freedom as
\[
  \mathcal I := \{\,i : \mathbf x_i \notin \Gamma_D\,\}, \qquad
  \mathcal B_D := \{\,i : \mathbf x_i \in \Gamma_D\,\},
\]
so that the functions \(\{\phi_i\}_{i\in\mathcal I}\) form a basis of
\(V_{h,0}\).
Denoting \(N_0 := |\mathcal I|\) and, without loss of generality,
relabeling the indices so that \(\mathcal I = \{1,\dots,N_0\}\),
we obtain that \(\{\phi_i\}_{i=1}^{N_0}\) is a nodal basis of \(V_{h,0}\). 
Define the assembled matrices and load by
\[
  [\mathbf A(\mathbf k)]_{ij} := a(\phi_j,\phi_i;\mathbf k),\qquad
  [\mathbf F(\mathbf k)]_{i} := \ell(\phi_i;\mathbf k),\qquad i,j=1,\dots,N_h,
\]
and write the block partition
\[
  \mathbf A(\mathbf k)=
  \begin{bmatrix}
    \mathbf A_{II}(\mathbf k) & \mathbf A_{I B_D}(\mathbf k)\\
    \mathbf A_{B_D I}(\mathbf k) & \mathbf A_{B_D B_D}(\mathbf k)
  \end{bmatrix},\qquad
  \mathbf F(\mathbf k)=\begin{bmatrix} \mathbf F_I(\mathbf k)\\ \mathbf F_{B_D}(\mathbf k)\end{bmatrix}.
\]
The discrete lifting $\mathbf E_h:\mathbb R^{|\mathcal B_D|}\to\mathbb R^{N_h}$ corresponding to \eqref{eq:E_D} is given by
\[
  \mathbf E_h \mathbf g_{\mathcal B_D} :=
  \begin{bmatrix}
    \mathbf L_{\mathrm{int}}\, \mathbf g_{\mathcal B_D}\\[2pt]
    \mathbf g_{\mathcal B_D}
  \end{bmatrix},\quad\text{with}\quad \mathbf L_{\mathrm{int}}:=-\big(\mathbf A^\star_{II}\big)^{-1}\mathbf A^\star_{I B_D}
  .
\]
Here, $\mathbf g_{\mathcal B_D}$ is the Dirichlet data vector on $\mathcal B_D$ and $\mathbf A^\star_{II}$, $\mathbf A^\star_{I B_D}$ are blocks in $\mathbf A^\star:=\mathbf A(\mathbf k_\star)$.
The Dirichlet boundary metric induced by the reference energy is
\[
  \mathbf W_\Gamma \;:=\; \mathbf E_h^\top \mathbf A^\star \mathbf E_h ,
\]
which yields the discrete inner product 
$\langle \boldsymbol\xi,\boldsymbol\zeta\rangle_{D,\star}:=\boldsymbol\xi^\top \mathbf W_\Gamma \boldsymbol\zeta$
on boundary vectors.
The discrete form of \eqref{eq:ps:aggload} is given by
\[
  \widehat{\mathbf {F}}_I(\mathbf k)
  := \mathbf F_I(\mathbf k)
     - \mathbf A_{I I}(\mathbf k)\,\mathbf L_{\mathrm{int}}\,\mathbf g_{\mathcal B_D}(\mathbf k)
     - \mathbf A_{I B_D}(\mathbf k)\,\mathbf g_{\mathcal B_D}(\mathbf k).
\]
The unknown $\mathbf w_I(\mathbf k)\in\mathbb{R}^{N_0}$ solves 
\begin{equation}\label{eq:discrete w}
   \mathbf A_{I I}(\mathbf k)\,\mathbf w_I(\mathbf k) \;=\; \widehat{\mathbf {F}}_I(\mathbf k),   
\end{equation}
and  
\[\mathbf u_h(\mathbf k)=\begin{bmatrix} \mathbf w_I(\mathbf k)\\ \mathbf{0}\end{bmatrix}+\mathbf E_h \mathbf g_{\mathcal B_D}(\mathbf k).
\]
Since $V_{h,0}\subset V_0$, the continuity and coercivity constants of $a(\cdot,\cdot;\mathbf k)$
carry over to $V_h$ uniformly in $h$ and $\mathbf k$. Hence, the discrete problem is well-posed.
We assume quadratures of sufficient order such that numerical integration errors are negligible.

\section{RB--DeepONet framework}\label{sec:regimes}

Throughout the paper, we work with the homogeneous, lifted formulation
\eqref{eq:ps:hom}:
for given $\mathbf k\in\mathcal D$, find $w(\mathbf k)\in V_0$ such that
\[
  a\big(w(\mathbf k),v;\mathbf k\big)=\mathfrak F(\mathbf k)[v]\qquad\forall v\in V_0,
\]
with $\mathfrak F(\mathbf k)$ defined in \eqref{eq:ps:aggload}.  The full solution is
$u(\mathbf k)=w(\mathbf k)+E_D[g_D(\mathbf k)]$.
We present a unified RB--DeepONet framework that fixes an RB trunk space and learns the branch coefficients via residual minimization with an offline/online split.
Within this framework, we focus on two specific problem cases:

\begin{itemize}
    \item \textbf{Case~I: fully parameterized operators and data.} There exist known maps
\[
  \mathbf k\mapsto \big\{A(\cdot;\mathbf k), c_0(\cdot;\mathbf k), \beta(\cdot;\mathbf k),f(\cdot;\mathbf k),\,g_D(\cdot;\mathbf k),\,h_N(\cdot;\mathbf k),\,r_R(\cdot;\mathbf k)\big\}.
\]
Hence, the learning target is the solution operator 
\begin{equation}\label{eq:map1}
    \mathbf k\ \longmapsto\ w(\cdot;\mathbf k).
\end{equation}
\item \textbf{Case~II: parametric operators with independently varying data.} Here, we consider a more complicated case, where 
only the coefficients depend on $\mathbf k$, while the data may vary independently.
We therefore allow the map
\begin{equation*}
    \begin{aligned}
     \mathbf k&\mapsto \big\{A(\cdot;\mathbf k), c_0(\cdot;\mathbf k), \beta(\cdot;\mathbf k)\big\}.
    \end{aligned}
\end{equation*}
To avoid handling full-order fields online, we compress  exogenous data offline into
low-dimensional coordinates:
\begin{itemize}
  \item source coefficients $\mathbf a\in\mathbb R^{r_f}$ obtained by projecting the
  aggregated load onto \emph{source modes};
  \item boundary coefficients $\mathbf b\in\mathbb R^{r_g}$ obtained by projecting
  $g_D$ onto \emph{boundary modes}.
\end{itemize}
The online learning target is then
\begin{equation}\label{eq:map2}
   \mathbf{k}_{\rm aug}:=\big[\mathbf k^\top,\ \mathbf a^\top,\ \mathbf b^\top\big]^\top
  \ \longmapsto\ w(\cdot;\mathbf{k}_{\rm aug}).   
\end{equation}
\end{itemize}

Since the FE dimension $N_h$ is typically large, solving the full system \eqref{eq:discrete w} repeatedly for many
$\mathbf{k}\in\mathcal D$ becomes prohibitively expensive. 
Both \eqref{eq:map1} and \eqref{eq:map2} are
instances of operator learning, which motivates us to consider using DeepONet to solve them. 
A naive DeepONet would treat the
trunk as an unconstrained black box: the trunk provides learned basis functions, the
branch produces coefficients, and the prediction for $w(\mathbf{k})$ is their inner product, i.e., a linear combination of the trunk functions with branch coefficients. 
While flexible, this representation does not encode any of the underlying PDE or variational structure, and offers little \emph{a priori} control over stability or approximation quality.
A more structured idea is therefore to use basis functions that already embed information about the solution manifold as the trunk. At one extreme, one could take the full finite element basis $\{\phi_i\}_{i=1}^{N_0}$ as the trunk \cite{lee2025finite}, but this would require predicting $N_0$ coefficients, which is both computationally prohibitive and unnecessary when the solution manifold is low-dimensional. POD--DeepONet \cite{lu2022comprehensive} implements a related idea by replacing the FE basis with a POD trunk, but it still relies on a large ensemble of high-fidelity snapshots to construct that trunk space.

RB spaces are rigorously validated in numerical analysis, which capture parametric solution manifolds with very few modes, and admit both \emph{a priori} and \emph{a posteriori} error estimates. Motivated by these properties, we adopt a reduced-order operator-learning pipeline and propose the RB--DeepONet framework, which couples the DeepONet architecture with RB methodology.
Concretely, we construct a low-dimensional reduced space
\[
V_{\mathrm{rb}}:=\operatorname{span}\{\psi_1,\ldots,\psi_N\}\subset V_0,\qquad N\ll N_h,
\]
via Greedy selection. In RB--DeepONet, $V_{\mathrm{rb}}$ acts as a fixed, interpretable trunk and the branch network predicts RB coefficients $\mathbf c_\theta(\mathbf k)$, i.e.,
\begin{equation}\label{eq:map11}
    \boldsymbol{\cdot} \;\longmapsto\; \mathbf c_\theta(\boldsymbol{\cdot}),
\end{equation}
where $\boldsymbol{\cdot}=\mathbf{k}$ in Case~I and $\boldsymbol{\cdot}=\mathbf{k}_{\rm aug}$ in Case~II.
In practice, the branch network can be implemented using standard architectures such as CNNs, MLPs, or ResNets.
In our experiments, the specific architecture is described in Section~\ref{sec:Numerical Results}.

Combining the fixed trunk and the output of branch network, the RB--DeepONet prediction is expressed as
\begin{equation}\label{eq:deeponet}
    w_\theta(\mathbf{x};\boldsymbol{\cdot})
  := \sum_{i=1}^N c_{\theta,i}(\boldsymbol{\cdot})\,\psi_i(\mathbf{x})
  = \Psi(\mathbf{x}) \mathbf c_\theta(\boldsymbol{\cdot}).
\end{equation}
where $\Psi(\mathbf{x})=[\psi_1(\mathbf{x}),\ldots,\psi_N(\mathbf{x})]$ denotes the fixed trunk basis, and $\mathbf c_\theta=[c_{\theta,1},\cdots,c_{\theta,N}]^\top\in\mathbb{R}^N$ is the branch network output.
This yields an offline–online split whose inference cost depends on the RB dimension $N$ but not on
$N_h$, thereby avoiding repeated full-order solves while preserving quantitative accuracy.

For both cases, the RB-trunk construction and the loss function are identical. 
The only differences arise in the inputs of the branch and in the assembly of the aggregated load \eqref{eq:ps:aggload} during the offline stage. 
Accordingly, we present the shared material once (under Case~I) in Sections~\ref{sec:rb-construction}-\ref{sec:Training and inference}, and introduce the specific treatment for Case~II in Section~\ref{sec:Specialization for Case2}.

\subsection{RB trunk construction}\label{sec:rb-construction}

Our goal is to build a low-dimensional space \(V_{\mathrm{rb}}\subset V_0\) from a
parameter sample set \(\mathcal S:=\{\mathbf k_i\}_{i=1}^{N_k}\subset\mathcal D\) such that the
Galerkin projection on \(V_{\mathrm{rb}}\) provides accurate surrogates uniformly over
\(\mathcal S\).
We adopt a Greedy selection as the default strategy: the space is grown
iteratively, one truth snapshot per iteration, until a certified error indicator
falls below a prescribed tolerance.
Thus, it requires only \(N\) high-fidelity solves to
build an \(N\)-dimensional RB space. For comparison, we also summarize the POD construction, which needs all \(N_k\) snapshots but enjoys an optimality identity. Further details can be found in \cite{hesthaven2016certified}.

First, we choose an initial sample \(\mathbf k_1\in\mathcal S\), compute
the high-fidelity solution \(w(\mathbf k_1)\),  normalized in the \(V\)-norm, and initialize
\(V_{\mathrm{rb}}=\operatorname{span}\{w(\mathbf k_1)\}\).
Given the current space \(V_{\mathrm{rb}}\), for each \(\mathbf k\in\mathcal S\), we compute the
reduced approximation \(w_{\mathrm{rb}}(\mathbf k)\in V_{\mathrm{rb}}\) as the solution of
\begin{equation}\label{eq:rb-variational}
  a\big(w_{\mathrm{rb}}(\mathbf k),v;\mathbf k\big)\;=\;\mathfrak F(\mathbf k)[v]
  \qquad \forall v\in V_{\mathrm{rb}}.
\end{equation}
Define the residual functional and its dual norm by
\begin{equation}\label{eq:rb-residual}
  r(v;\mathbf k):=\mathfrak F(\mathbf k)[v]-a\big(w_{\mathrm{rb}}(\mathbf k),v;\mathbf k\big),
  \qquad
  \|r(\cdot;\mathbf k)\|_{V_0'}:=\sup_{0\neq v\in V_0}\frac{r(v;\mathbf k)}{\|v\|_V}.
\end{equation}
With the coercivity lower bound \(\alpha_{\mathrm{LB}}>0\),
the \emph{a posteriori} estimator
\begin{equation}\label{eq:rb-estimator}
  \eta(\mathbf k):=\alpha_{\mathrm{LB}}^{-1}\|r(\cdot;\mathbf k)\|_{V_0'}
\end{equation}
controls the energy error \cite{hesthaven2016certified}:
\begin{equation}\label{eq:rb-reliability}
  \|w(\mathbf k)-w_{\mathrm{rb}}(\mathbf k)\|_V \;\le\; \eta(\mathbf k)
  \qquad \forall\,\mathbf k\in\mathcal D.
\end{equation}
Thus, we use $\eta(\mathbf k)$ to do Greedy selection as summarized in Algorithm \ref{alg:greedy} below.

\begin{algorithm}[H]\label{alg:greedy}
  \caption{Greedy construction of \(V_{\mathrm{rb}}\)}
  \KwIn{parameter sample set \(\mathcal S\), tolerance \(\epsilon_{\mathrm{greedy}}>0\)}
  \KwOut{RB basis \(\{\psi_i\}_{i=1}^{N}\)}
  Set $n=1$.\\
  Pick \(\mathbf k_1\in\mathcal S\); compute the truth \(w(\mathbf k_1)\), normalize in \(\|\cdot\|_V\),
  and set \(V_{\mathrm{rb}}=\mathrm{span}\{w(\mathbf k_1)\}\).\\
  \While{\(\max_{\mathbf k\in\mathcal S}\eta(\mathbf k)>\epsilon_{\mathrm{greedy}}\)}{
    For each \(\mathbf k\in\mathcal S\), solve \eqref{eq:rb-variational} in \(V_{\mathrm{rb}}\) and evaluate \(\eta(\mathbf k)\).\\
    Set \(\mathbf k_{n+1}=\arg\max_{\mathbf k\in\mathcal S}\eta(\mathbf k)\).\\
    Compute the truth \(w(\mathbf k_{n+1})\); \(V\)-orthonormalize and enrich
    \(V_{\mathrm{rb}}\leftarrow V_{\mathrm{rb}}\oplus\mathrm{span}\{w(\mathbf k_{n+1})\}\).\\
    Set $n=n+1$.
  }
\end{algorithm}

As summarized in Algorithm~\ref{alg:greedy}, the offline cost is
proportional to \(N\) truth solves.
Under an affine decomposition or an empirical interpolation surrogate, all
parameter-independent contributions to \eqref{eq:rb-variational} and
\eqref{eq:rb-estimator} can be preassembled, so the online evaluation of
\(\eta(\mathbf k)\) and the reduced solve scales only with
\(\dim V_{\mathrm{rb}}=N\), not with the full FE dimension.
Moreover, the Greedy algorithm uses the estimator \eqref{eq:rb-estimator}
as its selection criterion, so the maximal indicator over the sample set
decreases monotonically with each enrichment step, providing a built-in
measure of convergence of the RB space.

Another way to construct $V_{\mathrm{rb}}$ is POD.
Let \(w^i:=w(\mathbf k_i)\in V_{0}\) be the snapshots and
\(V_{\mathcal{S}}=\mathrm{span}\{w^i: i=1,\dots,N_k\}\). Define the correlation operator
\begin{equation}\label{eq:pod-corr}
  \mathcal C(v):=\tfrac1{N_k}\sum_{i=1}^{N_k}(v,w^i)_V\,w^i,\qquad v\in V_0,
\end{equation}
and let \(\{(\lambda_i,\psi_i)\}_{i=1}^{N_k}\) be its eigenpairs in \(V_{\mathcal{S}}\),
ordered \(\lambda_1\ge\cdots\ge\lambda_{N_k}\ge0\) with \(\|\psi_i\|_V=1\).
The POD space is \(V_{\mathrm{rb}}=\mathrm{span}\{\psi_1,\dots,\psi_N\}\), and it satisfies
the optimality identity
\begin{equation}\label{eq:pod-opt}
  \frac1{N_k}\sum_{i=1}^{N_k}\|w^i-P_N w^i\|_V^2
  \;=\;\sum_{i=N+1}^{N_k}\lambda_i,
\end{equation}
where \(P_N\) is the \(V\)-orthogonal projector onto \(V_{\mathrm{rb}}\).
In practice, \(N\) is chosen as the smallest integer such that
\begin{equation}\label{eq:pod-energy}
  1-\frac{\sum_{i=1}^{N}\lambda_i}{\sum_{i=1}^{N_k}\lambda_i}\ \le\ \epsilon_{\mathrm{POD}}^{\,2},
\end{equation}
for some tolerance $\epsilon_{\mathrm{POD}}>0$.

Note that POD requires all \(N_k\) snapshots but delivers the best mean-square projection in
\eqref{eq:pod-opt}. In contrast, the Greedy procedure attains comparable accuracy with
significantly fewer truth solves and, with \eqref{eq:rb-reliability}, provides certified
selection on \(\mathcal S\).

\begin{remark}\label{remark: comparsion1}
In addition to Greedy selection, we also include POD, since it is the standard mechanism for prescribing a fixed trunk in POD--DeepONet \cite{lu2022comprehensive}. In that setting, one first computes a set of high-fidelity snapshots, forms the POD modes, and then fixes the trunk to the leading modes while the branch network predicts their coefficients. In our numerical experiments in Section \ref{sec:Numerical Results}, the POD trunk used for POD--DeepONet is constructed exactly in this way.

In RB--DeepONet, POD is only one possible choice for constructing an interpretable trunk. A certified Greedy procedure can generate a basis of the same size by adaptively sampling parameters using \emph{a posteriori} error estimators, thereby reducing the offline cost while retaining rigorous error control \cite{hesthaven2016certified}. Our numerical results indicate that, for a fixed reduced dimension $N$, POD- and Greedy-generated trunks yield comparable prediction errors. We therefore report both, using the Greedy basis as the default for offline economy and the POD basis as a convenient baseline.
\end{remark}

\subsection{Offline reduced operators}\label{sec:rb-assembly}

Let $\{\psi_i\}_{i=1}^N\subset V_0$ be the reduced basis built by Algorithm \ref{alg:greedy}.
Each $\psi_i$ is represented in the FE trial space $V_{h,0}$ with basis
$\{\phi_j\}_{j=1}^{N_0}$. We denote this FE representation by $\psi_{h,i}\in V_{h,0}$.
Thus there exist coefficient vectors $\mathbf v_i\in\mathbb R^{N_0}$ such that
\begin{equation}\label{eq:psi-FE-expansion}
  \psi_{h,i}(\mathbf x)=\sum_{j=1}^{N_0} (\mathbf v_i)_j\,\phi_j(\mathbf x),
  \qquad i=1,\ldots,N.
\end{equation}
Collecting the columns $\mathbf v_i$ yields the reduced basis matrix
\begin{equation}\label{eq:rb-matrix-Psi}
    \Psi := [\,\mathbf v_1,\ldots,\mathbf v_N\,]\in\mathbb R^{N_0\times N}, \qquad N\ll N_0.
\end{equation}
Recall that the full FE solution coefficients
$\mathbf w_{I}(\mathbf k)\in\mathbb R^{N_0}$ satisfy \eqref{eq:discrete w}.
The RB--Galerkin approximation seeks $\mathbf c_N(\mathbf k)=[c_{N,1}(\mathbf k),\ldots,c_{N,N}(\mathbf k)]^\top\!\in\mathbb R^N$
such that $\Psi\,\mathbf c_N(\mathbf k)\approx \mathbf w_{I}(\mathbf k)$. Galerkin projection on
$V_{\rm rb}=\mathrm{span}\{\psi_i\}_{i=1}^N$ gives the reduced system
\begin{equation}\label{eq:rb-system}
  \mathbf A_{\mathrm{rb}}(\mathbf k)\,\mathbf c_N(\mathbf k)=\mathbf F_{\mathrm{rb}}(\mathbf k),
\end{equation}
with
\begin{equation}\label{eq:rb-operators}
  \mathbf A_{\mathrm{rb}}(\mathbf k)=\Psi^\top \mathbf A_{II}(\mathbf k)\Psi,
  \qquad
  \mathbf F_{\mathrm{rb}}(\mathbf k)=\Psi^\top \widehat{\mathbf{F}}_I(\mathbf k).
\end{equation}
Under the uniform well-posedness assumptions of the full model, the RB--Galerkin approximation \eqref{eq:rb-system} is
well-defined and admits a unique solution
$\mathbf c_N(\mathbf k)$, which is the coefficient vector that our branch network
aims to predict.

To make the online cost independent of $N_0$, we employ an affine decomposition of the FE
operators:
\begin{equation}\label{eq:affine-decomp}
  \mathbf A_{II}(\mathbf k)=\sum_{p=1}^{Q_a}\Theta^a_p(\mathbf k)\,\mathbf A_p,
  \qquad
  \widehat{\mathbf{F}}_{I}(\mathbf k)=\sum_{q=1}^{Q_f}\Theta^f_q(\mathbf k)\,\mathbf F_q,
\end{equation}
where $\Theta^a_p,\Theta^f_q$ are parameter-dependent scalars and
$\mathbf A_p,\mathbf F_q$ are parameter-independent FE quantities.
In the offline stage, we precompute and store the reduced components
\begin{equation}\label{eq:A_p F_q}
      \mathbf A^N_p:=\Psi^\top \mathbf A_p\Psi,\qquad
  \mathbf F^N_q:=\Psi^\top \mathbf F_q,
  \qquad p=1,\ldots,Q_a,\; q=1,\ldots,Q_f.
\end{equation}
Then, at query time, the reduced system \eqref{eq:rb-system} is assembled as
\begin{equation}\label{eq:A_rb 1}
      \mathbf A_{\mathrm{rb}}(\mathbf k)=\sum_{p=1}^{Q_a}\Theta^a_p(\mathbf k)\,\mathbf A^N_p,
  \qquad
  \mathbf F_{\mathrm{rb}}(\mathbf k)=\sum_{q=1}^{Q_f}\Theta^f_q(\mathbf k)\,\mathbf F^N_q,
\end{equation}
at cost $\mathcal O(Q_a N^2+Q_f N)$, independent of $N_0$.
If an exact affine form is not available, empirical interpolation (EIM/DEIM)
can be used to approximate \eqref{eq:affine-decomp}~\cite{hesthaven2016certified,chaturantabut2010nonlinear,chaturantabut2009discrete}.
This non-affine setting is instantiated in our numerical experiments (Example~\ref{sec:ex3}), where EIM is employed to approximate the parameter dependence, thereby demonstrating the broad applicability of the proposed offline–online framework.

\subsection{Training and inference}\label{sec:Training and inference}

In the online stage, we feed the branch network with $\mathbf{k}$ and obtain the RB coefficients
$\mathbf c_\theta(\mathbf k)\in\mathbb R^N$.
With the fixed trunk $\Psi(\mathbf{x})$,
the RB--DeepONet reconstructions are given in \eqref{eq:deeponet}, and the prediction reads
\[
u_\theta(\mathbf x;\mathbf k)
=w_\theta(\mathbf x;\mathbf k)
+E_D(g_D(\mathbf x;\mathbf k)).
\]
Rather than data-driven training, 
we employ the variational formulation and enforce physics through the residual projected onto the RB space.
Using the RB stiffness $\mathbf{A}_{\rm rb}(\mathbf{k})$ and load $\mathbf{F}_{\rm rb}(\mathbf{k})$ defined in \eqref{eq:A_rb 1}, we collect the RB residual tested against $\{\psi_i\}_{i=1}^N$ in the vector
\begin{equation}\label{eq: rb residual}
    \mathbf{r}(\mathbf{k}) := \mathbf{F}_{\rm rb}(\mathbf{k})-\mathbf{A}_{\rm rb}(\mathbf{k})\mathbf{c}_\theta(\mathbf{k})\in\mathbb R^N.
\end{equation}
Given a parameter set $\mathcal B_s=\{\mathbf{k}_i\}_{i=1}^{N_s}$, the loss function is then defined as
\begin{equation}\label{eq:forward-loss}
    \mathcal{L}(\theta)=\frac{1}{N_{s}}\sum_{i=1}^{N_{s}}\|\mathbf A_{\rm rb}(\mathbf k_i)^{-1/2}\mathbf r(\mathbf{k_i})\|_2^2,
\end{equation}
This is exactly the intrusive Galerkin equilibrium condition projected onto the reduced basis.
Thus, minimizing \eqref{eq:forward-loss} drives the RB-projected residual to zero. In particular, if $\mathcal L(\theta)=0$, then $\mathbf{r}(\mathbf{k}_i)=\mathbf{0}$, and $\mathbf{c}_\theta(\mathbf{k}_i)$ coincides with the RB--Galerkin coefficients $c_N(\mathbf{k}_i)$ defined in \eqref{eq:rb-system}, for all $i=1,\cdots,N_s$.

The overall procedure for RB--DeepONet is summarized in Algorithm \ref{alg: RB--DeepONet} below.

\begin{algorithm}[htbp]\label{alg: RB--DeepONet}
  \caption{RB--DeepONet for parametric PDEs (fixed trunk, learned branch)}
  \KwIn{Parameter samples $\mathcal{S}$; epochs $T$; batch size $N_s$;  
        Greedy tolerance $\epsilon_{\rm greedy}$.}
  \KwOut{Trained surrogate $u_\theta(\mathbf{x};\mathbf{k})$.}

  \BlankLine
  \textbf{Offline (RB-trunk and reduced operators).}\\
  Construct RB basis $\{\psi_i\}_{i=1}^N$ by Algorithm~\ref{alg:greedy}, 
         and form $\Psi(\mathbf{x}):=[\psi_1(\mathbf{x}),\ldots,\psi_N(\mathbf{x})]$.\\
         Preassemble all parameter-independent reduced operators

  \BlankLine
  \textbf{Training (operator learning).}\\
  Fix $\Psi(\mathbf{x})$ as trunk; initialize branch net $\mathbf c_\theta$.\\
  \For{$t=1$ \KwTo $T$}{
        Shuffle $\mathcal{S}$ and split it into batches $\{\mathcal B_s\}_{s=1}^S$ of size $N_s$.\\
        \For{$s=1$ \KwTo $S$}{
        Update the branch network by minimizing $\mathcal{L}(\theta)$ defined in \eqref{eq:forward-loss}.
        }
     }

  \BlankLine
  \textbf{Online (prediction).}\\
  Given new $\mathbf{k}\in\mathcal D$, output $w_\theta(\mathbf{x};\mathbf{k})=\Psi(\mathbf{x})  \mathbf c_\theta(\mathbf{k})$.\\ 
  Formulate $u_\theta(\mathbf{x};\mathbf k)=w_\theta(\mathbf{x};\mathbf k)+E_D[g_D(\mathbf{x};\mathbf k)]$.
\end{algorithm}

\subsection{Extension to independent boundary and source data}\label{sec:Specialization for Case2}

Recall that in Case II we assume only the operator coefficients $A(\cdot;\mathbf k), c_0(\cdot;\mathbf k), \beta(\cdot;\mathbf k)$ depend on $\mathbf k$, whereas the data $(f,g_D,h_N,r_R)$ are allowed to vary inpedently.
The construction of trunk basis has the same procedure as we introduced in Section \ref{sec:rb-construction} and the training procedure is identical to Algorithm~\ref{alg: RB--DeepONet}. The only
changes are the branch input and the assembly of the reduced right–hand side $\mathbf F_{\mathrm{rb}}(\mathbf{k})$.

\subsubsection{Boundary and source modes construction}\label{sec:bd-src-construct}

To avoid carrying full-order fields online, we construct, in the offline stage, two compact data-driven spaces from snapshot ensembles:
a boundary space
$E_{r_g}\subset G_D$ for Dirichlet traces and a source space $W_{r_f}\subset V_0$
for the aggregated load functional.
At inference time, we work only with the modal coordinates $\mathbf{b}\in\mathbb{R}^{r_g}$ and $\mathbf{a}\in\mathbb{R}^{r_f}$, obtained by projecting the data onto $E_{r_g}$ and $W_{r_f}$.
We adopt a certified Greedy construction.
As in Algorithm \ref{alg:greedy}, we start from a rank–one initial space by a randomly chosen function in the corresponding function space. At each iteration, we select
modes by maximizing certified error indicators

Let $\mathcal G=\{g^{(m)}\}_{m=1}^{N_g}\subset G_D$ be a training set of boundary traces. Given the current boundary subspace $E_r\subset G_D$ and its
orthogonal projector $P_{E_r}$ in $\langle\cdot,\cdot\rangle_{D,\star}$, we define the \emph{a posteriori} estimator
\begin{equation}\label{eq:bd-greedy-ind}
  \eta_g(m):=\big\|\,g^{(m)}-P_{E_r}g^{(m)}\,\big\|_{D,\star}.
\end{equation}
The Greedy update selects $m^\star=\arg\max_{m}\eta_g(m)$.
The corresponding function $g^{(m^\star)}$ is orthonormalized in $\langle\cdot,\cdot\rangle_{D,\star}$ and appended to the boundary space $E_{r+1}$.
The procedure is repeated until $\max_m\eta_g(m)\le\epsilon_g$ for a prescribed tolerance $\epsilon_g$.
The resulting \emph{boundary modes} are denoted by $\{\eta_n\}_{n=1}^{r_g}$, which span the space $E_{r_g}$. Their corresponding liftings are defined as $\mathcal{E}_n := E_D[\eta_n]$.

For any parameter $\mathbf k$, let $u(\mathbf k)$ be the solution with datum $g^{(m)}$,
and let $u^{(g_r)}(\mathbf k)$ be the solution with $g_r:=P_{E_r}g^{(m)}$.
By the lifting isometry and the coercivity of $a(\cdot,\cdot;\mathbf k)$, there exists
$C_g:=1+M\alpha_{\mathrm{LB}}^{-1}$ such that
\begin{equation}\label{eq:bd-greedy-reliable}
  \big\|u(\mathbf k)-u^{(g_r)}(\mathbf k)\big\|_V
  \ \le\ C_g\,\eta_g(m) \qquad \forall\,\mathbf k\in\mathcal{D},\ \forall\,m\in\{1,\cdots,N_g\} .
\end{equation}

To obtain \emph{source modes}, let $\mathcal F=\{\mathfrak F^{(n)}\}_{n=1}^{N_f}\subset V_0'$ be a training set
of aggregated load functionals, which is built from representative $f,g_D,h_N,r_R$.
We define the Riesz map
$\mathcal R_\star:V_0\to V_0'$ by
\begin{equation*}
    (\mathcal R_\star q)[v]:=a(q,v;\mathbf k_\star)=(q,v)_V.
\end{equation*}
Following Riesz representation theorem \cite{brezis2011functional}, $\mathcal R_\star$
is an isometric isomorphism.
Hence, for each aggregated load functional
$\mathfrak F^{(n)}\in V_0'$, there exists a unique Riesz representer
$q^{(n)}=\mathcal R_\star^{-1}\mathfrak F^{(n)}\in V_0$ satisfying
\begin{equation}\label{eq:ref-q}
  a\big(q^{(n)},v;\mathbf k_\star\big)=\mathfrak F^{(n)}[v]\qquad \forall v\in V_0.
\end{equation}
Approximating $\mathfrak F^{(n)}$ in the dual norm $\|\cdot\|_{V_0'}$ is therefore
equivalent to approximating its representer $q^{(n)}$ in the $V$–norm:
\[
\min_{w\in W_r}\|\mathfrak F^{(n)}-\mathcal R_\star w\|_{V_0'}
\;=\; \min_{w\in W_r}\|q^{(n)}-w\|_V .
\]
Thus, we project functionals by first mapping them to $q$ via
\eqref{eq:ref-q} and then working in the primal space $V_0$.
Let $W_r\subset V_0$ be the current source subspace and
$P_{W_r}$ the $(\cdot,\cdot)_V$–orthogonal projector, we define the \emph{a posteriori} estimator
\begin{equation}\label{eq:src-greedy-ind}
  \eta_f(n):=\big\|\,q^{(n)}-P_{W_r}q^{(n)}\,\big\|_V =\min_{w\in W_r}\big\|q^{(n)}-w\big\|_V
  =\min_{w\in W_r}\big\|\mathfrak F^{(n)}-\mathcal R_\star w\big\|_{V_0'} .
\end{equation}
Then, $\eta_f(n)$ is precisely the best-approximation error of the functional in
$\|\cdot\|_{V_0'}$. 
The greedy update selects
$n^\star=\arg\max_n\eta_f(n)$.
The corresponding residual is orthonormalized in $V$–norm and appended to $W_r$.
The iteration terminates once $\max_n\eta_f(n)\le\epsilon_f$ for a prescribed tolerance $\epsilon_f$.
The resulting source modes are $\{W_f^{(m)}\}_{m=1}^{r_f}$, which span $W_{r_f}$.

For any parameter $\mathbf k$, let $w(\mathbf k)\in V_0$ solve the homogeneous problem
with load $\mathfrak F^{(n)}$ and let $w_r(\mathbf k)\in V_0$ solve the same
problem with projected load $(P_{W_r}q^{(n)},\cdot)_V$.
Then, using the uniform coercivity of $a(\cdot,\cdot;\mathbf{k})$, we obtain the reliability bound
\begin{equation}\label{eq:src-greedy-reliable}
  \big\|w(\mathbf k)-w_r(\mathbf k)\big\|_V
  \ \le\ \alpha_{\mathrm{LB}}^{-1}\,\eta_f(n)
  \qquad \forall\,\mathbf k\in\mathcal{D},\ \forall\,n\in\{1,\cdots,N_f\}.
\end{equation}

Combining \eqref{eq:bd-greedy-reliable} and \eqref{eq:src-greedy-reliable} shows that,
for tolerances $\epsilon_g,\epsilon_f$, the boundary and source projections contribute at most 
$\mathcal O(\epsilon_g)$ and $\mathcal O(\epsilon_f)$, respectively, to the $V$–norm error for all $k\in D$ and for all training snapshots $g^{(m)}\in G$, $F^{(n)}\in F$.
Consequently, we feed only the projection coordinates into the network, thereby avoiding any
manipulation of full-order vectors or matrices during training and inference.
On these offline snapshot sets, the induced modeling error is bounded by
$\mathcal O(\epsilon_g+\epsilon_f)$ and becomes negligible once $r_g,r_f$ are chosen sufficiently large.

\begin{remark}[Projection error for unseen data]
For a new boundary datum $\tilde g\in G_D$ and a new aggregated load
$\tilde F\in V_0'$, the modeling error induced by the boundary and source projections is exactly the
distance of the data to the spaces $E_{r_g}$ and $W_{r_f}$.
If the admissible boundary and source data lie in the closures of
$\mathrm{span}(\mathcal G)$ and $\mathrm{span}(\mathcal F)$ (for instance when they are generated by
smooth parametric families and the snapshot sets sample the corresponding parameter
domains densely), then increasing $r_g$ and $r_f$ drives these distances, and hence the
projection error, to zero.
In our numerical experiment (Example \ref{sec:ex2}), all evaluation data are drawn from the same data
families as the offline snapshots, so the observed projection errors are of the same
order as predicted above.
\end{remark}


\subsubsection{Modal encoding and RB right-hand-side assembly}

By the discussion above, we construct a boundary subspace 
$E_{r_g}=\operatorname{span}\{\eta_n\}_{n=1}^{r_g}\subset G_D$ 
and a source subspace 
$W_{r_f}=\operatorname{span}\{W_f^{(m)}\}_{m=1}^{r_f}\subset V_0$.
In the online stage, rather than supplying the branch with full-order data $(f,g_D,h_N,r_R)$,
we use their modal coordinates obtained by projecting $g_D$ and the source functional onto 
$E_{r_g}$ and $W_{r_f}$, respectively.
This encoding yields a strict offline/online split and substantially reduces memory and latency.
Although truncation of the data modes introduces an additional error, the \emph{a priori} estimates\eqref{eq:bd-greedy-reliable}, \eqref{eq:src-greedy-reliable}
show that it is controlled by the chosen dimensions $(r_g,r_f)$ and can be made arbitrarily small.

Define the vector $\mathbf b:=[b_1,\cdots,b_{r_g}]^\top\in\mathbb R^{r_g}$ by
\begin{equation}\label{eq:b-vector}
  b_n
  \;:=\;
  \langle\,g_D,\,\eta_n\,\rangle_{D,\star},
  \qquad n=1,\dots,r_g,
\end{equation}
i.e., the orthogonal projection of $g_D$ onto $E_{r_g}$ in $\langle\cdot,\cdot\rangle_{D,\star}$.
Since $\{W_f^{(m)}\}_{m=1}^{r_f}$ is orthonormal in $\langle\cdot,\cdot\rangle_V$, applying \eqref{eq:ref-q},
the coefficients $\mathbf a:=[a_1,\cdots,a_{r_f}]^\top\in\mathbb R^{r_f}$ are obtained by
\begin{equation}\label{eq:a-vector}
  a_m
  \;:=\;
  \mathfrak F(\mathbf k)\!\left[\,W_f^{(m)}\,\right],
  \qquad m=1,\dots,r_f .
\end{equation}

Recall that the reduced operator $\mathbf A_{\mathrm{rb}}(\mathbf k)$ is assembled via the affine/EIM expansion introduced earlier in \eqref{eq:A_rb 1}.
For Case II, the reduced right–hand side should make use of the source and boundary modes as well as their coordinate vectors:
\begin{equation}\label{eq:Frb-clean}
  \mathbf F_{\mathrm{rb}}(\mathbf k_{\rm aug})
  \;:=\;
  \underbrace{\sum_{m=1}^{r_f} a_m(\mathbf k)\,\mathbf F^{(s)}_{m,\mathrm{rb}}}_{\text{source contribution}}
  \;-\;
  \underbrace{\sum_{n=1}^{r_g} b_n\,
    \Big(\sum_{p=1}^{Q_a}\Theta^{(a)}_{p}(\mathbf k)\,\mathbf G^{(p)}_{n,\mathrm{rb}}\Big)}_{\text{Dirichlet lifting}},
\end{equation}
where each term is precomputable offline:
\begin{equation}\label{eq:Frb-building-blocks}
 \mathbf F^{(s)}_{m,\mathrm{rb}} \;:=\; \Psi^\top \mathbf F\!\big(W_f^{(m)}\big)\in\mathbb R^{N},
  \qquad
  \mathbf G^{(p)}_{n,\mathrm{rb}} \;:=\; \Psi^\top \mathbf A_p\,\mathbf {E}_n \in\mathbb R^{N}.
\end{equation}
Here, \(\mathbf {E}_n\in \mathbb R^{N_0}\) are vectors of the lifted boundary modes $\mathcal{E}_n$ on $\mathcal{I}$.
The vector \(\mathbf F(v)\in\mathbb R^{N_0}\) associated with any \(v\in V_0\) is the discrete Riesz image
\([\mathbf F(v)]_i := a(v,\phi_i;\mathbf k_\star)\),
and satisfies \(\Psi^\top \mathbf F(v)=\big(a(v,\psi_j;\mathbf k_\star)\big)_{j=1}^{N_0}\).
Then, the RB approximation solves
\begin{equation}\label{eq:rb-linear-system}
  \mathbf A_{\mathrm{rb}}(\mathbf k)\,c_N(\mathbf k_{\rm aug}) \;=\; \mathbf F_{\mathrm{rb}}(\mathbf k_{\rm aug}),
\end{equation}
and we target the RB coefficient vector \(c_N(\mathbf k)\in\mathbb R^N\).
Under the uniform continuity and coercivity assumptions on $a(\cdot,\cdot;\mathbf k)$, \eqref{eq:rb-linear-system} is well posed.

\begin{remark}
  The vector $F_{\mathrm{rb}}(\mathbf k_{\rm aug})$ in \eqref{eq:Frb-clean}
  should be interpreted as an approximate RB right-hand side in general.
The ``exact'' RB right-hand side $F_{\mathrm{rb}}^{\mathrm{exact}}(\mathbf k)$ depends on the complete data $(f,g_D,h_N,r_R)$.
$F_{\mathrm{rb}}(\mathbf k_{\rm aug})$ in \eqref{eq:Frb-clean} approximates the boundary and source contributions by the projections onto $E_{r_g}$ and $W_{r_f}$,
which can be assembled online from
the low-dimensional coordinates $\mathbf a$, $\mathbf b$, and the affine/EIM coefficients $\Theta^{(a)}_{p}(\mathbf k)$, without manipulating any full-order vectors or matrices.

  By the reliability bounds for the boundary and source projections,
  \eqref{eq:bd-greedy-reliable}–\eqref{eq:src-greedy-reliable}, we have an error bound of the form
    \[
    \big\|\mathbf c_N(\mathbf k_{\rm aug})
          - \mathbf c_N^{\mathrm{exact}}(\mathbf k)\big\|
    \;\le\; C_{\mathrm{rb}} \,(\epsilon_g+\epsilon_f),
  \]
 for some uniform constant $C_{\mathrm{rb}}>0$.
  Here, $\mathbf c_N^{\mathrm{exact}}(\mathbf k)$ denotes the RB
  coefficient vector obtained by solving the RB--Galerkin system with
  $F_{\mathrm{rb}}^{\mathrm{exact}}(\mathbf k)$.
  As $r_g,r_f$ increase and $\epsilon_g,\epsilon_f\to 0$, the reduced
  right-hand side and the solution of \eqref{eq:rb-linear-system} converge
  to those of the fully constrained RB--Galerkin system, while strict
  offline–online separation is preserved.
\end{remark}

We feed the branch network with the concatenated feature vector
\begin{equation}\label{eq:feature vector}
    \mathbf k_{\rm aug}=\big[\,\mathbf k^\top,\ \mathbf a^\top,\ \mathbf b^\top\,\big]^\top
  \in\mathbb R^{p+r_f+r_g},
\end{equation}
and obtain the RB coefficients
$\mathbf c_\theta(\mathbf{k}_{\rm aug})\in\mathbb R^N$.
The residual load $F_{\mathrm{rb}}(\mathbf{k}_{\rm aug})$ \eqref{eq:Frb-clean} is used to form the loss function \eqref{eq:forward-loss}.
With the fixed trunk $\Psi=\{\psi_j\}_{j=1}^N$ and precomputed liftings $\{\mathcal{E}_n\}_{n=1}^{r_g}$,
the RB--DeepONet prediction reads
\[
u_\theta(\mathbf x;\mathbf{k}_{\rm aug})
=w_\theta(\mathbf x;\mathbf{k}_{\rm aug})
+\sum_{n=1}^{r_g} b_n\,\mathcal{E}_n(\mathbf x).
\]

\subsection{Summary and comparison}

Figure~\ref{fig:RB--DeepONet workflow} summarizes the RB--DeepONet pipeline, highlighting the fixed RB
trunk, the residual-based training objective, and the Case~II feature extraction for
$\mathbf a,\mathbf b$. 
Table~\ref{tab:cases-feonet} positions our method against POD--DeepONet~\cite{lu2022comprehensive} and FEONet~\cite{lee2025finite}.
In POD--DeepONet, the trunk is prescribed by a POD basis computed from labeled solution snapshots, while the branch network takes as input a discrete representation of the PDE data (e.g., coefficient fields, forcing terms, or boundary traces), and outputs the coefficient vector $\mathbf{c}_{\theta}(\mathbf{k})\in\mathbb{R}^N$.
The network is trained in a fully supervised fashion,
\begin{equation}\label{eq:pod deeponet loss}
    \mathcal{L}_{\rm sup}
    =\frac{1}{N_s}\sum_{i=1}^{N_s}
    \big\|u_{\theta}(\cdot;\mathbf{k}_i)-u_h(\cdot;\mathbf{k}_i)\big\|^2_{L^2(\Omega)}.
\end{equation}
Dirichlet data are enforced exactly by a lifting function $E_D$, i.e., 
$u_\theta(\cdot;\mathbf{k})=\Psi(\mathbf{x})\,\mathbf{c}_\theta(\mathbf{k})+E_D(g_D)$.
FEONet keeps the full FE basis as trunk and predicts the FE coefficient vector 
$\mathbf{w}_\theta(\mathbf{k})\in\mathbb R^{N_0}$ for each input features (e.g., forcing and/or coefficient fields and boundary values), while minimizing the FE variational residual,
\begin{equation}\label{eq: feonet loss}
    \mathcal{L}(\theta)
    =\frac{1}{N_{s}}\sum_{i=1}^{N_{s}}
    \big\|\mathbf{F}_I(\mathbf k_i)-\mathbf A_{II}(\mathbf k_i)\,\mathbf{w}_\theta(\mathbf k_i)\big\|_2^2.
\end{equation}
This formulation enables unsupervised training and guarantees exact enforcement of Dirichlet boundary conditions at the nodal degrees of freedom.

As a comparison, we briefly summarize the complexity profiles of RB--DeepONet, POD--DeepONet, and FEONet.

\emph{Offline.}
All three methods start from the same full-order discretization and require assembling the FE stiffness matrices and load vectors (with affine or empirical interpolation decompositions). 
Beyond this baseline, POD--DeepONet solves the full-order problem for $N_k$ parameter samples to generate snapshots, and then performs a POD and projection to obtain the POD trunk and the associated reduced operators. 
Its offline cost is therefore dominated by the $N_k$ truth solves. 
RB--DeepONet instead requires only $N$ truth solves to construct the RB trunk and its reduced operators, so its additional offline cost scales with $N$.
The offline cost of FEONet is only dominated by the full-order assembly.

\emph{Training.}
Let $\mathrm{net}(M)$ denote the cost of one forward/backward pass of the
branch network with $M$-dimensional output, for a fixed hidden
architecture (same depth and hidden widths across methods).
RB--DeepONet evaluates the PDE residual entirely in the reduced space, so
the per-sample cost is $\mathrm{net}(N)+\text{RB}$,
where RB denotes the reduced residual in $\mathbb R^N$ (order $N^2$), with
an additional low-rank term $\mathcal O(N r_g + N r_f)$ in Case~II.
POD--DeepONet uses the same branch architecture but, at each step,
reconstructs the full field and compares it to FEM snapshots, leading to $\mathrm{net}(N)+\text{full}(N_0 N)$
per sample, where $\text{full}(N_0 N)$ summarizes the full-field
reconstruction and supervised loss on $N_0$ degrees of freedom.
FEONet directly predicts a full-order solution in $\mathbb R^{N_0}$ and
minimizes the FE variational residual, so its per-sample cost is $\mathrm{net}(N_0)+\text{FE}$,
with FE denoting a full-order FE residual evaluation.
For fixed depth and hidden widths, $\mathrm{net}(M)$ grows approximately
linearly with $M$, so $\mathrm{net}(N_0)$ is larger than
$\mathrm{net}(N)$ by a factor of order $N_0/N$.
Thus, in the typical regime $N \ll N_0$, FEONet is the most expensive to
train, POD--DeepONet lies in between, and RB--DeepONet is the cheapest.

\emph{Online prediction.}
Given a trained model, RB--DeepONet and POD--DeepONet evaluate only the
branch network $c_\theta(\mathbf k)\in\mathbb R^N$,
so the network evaluation cost scales as $\mathcal O(N)$ (or
$\mathcal O(N + r_g + r_f)$ in Case~II).
FEONet, in contrast, outputs a full-order vector $w_\theta(\mathbf k)\in\mathbb R^{N_0}$,
and its online cost scales linearly with $N_0$. Reconstructing the full field
$w_\theta(\mathbf k)=\Psi\,c_\theta(\mathbf k)$ adds a common
$\mathcal O(N_0 N)$ matrix–vector multiplication to both RB--DeepONet and
POD--DeepONet, and is therefore omitted from the network-cost comparison
in Table~\ref{tab:cases-feonet}.

\begin{figure}[htbp]
\centering
\includegraphics[width=.98\textwidth]{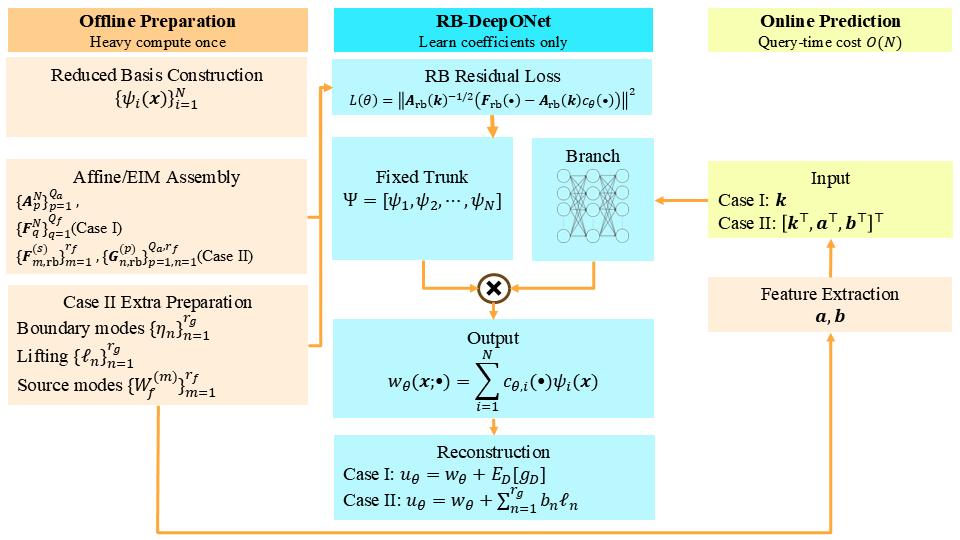}
\caption{RB--DeepONet workflow: offline preparation (Greedy RB basis,
affine precomputation, and boundary/source modes), residual-based training with a fixed
trunk and a branch predicting $\mathbf c_\theta$, and online prediction with inputs
$\mathbf k$ (Case~I) or $[\mathbf k,\mathbf a,\mathbf b]$ (Case~II). In the middle column, $\bullet=k$ (Case I), $\bullet=k_{\mathrm{aug}}$ (Case II).}
\label{fig:RB--DeepONet workflow}
\end{figure}

\begin{table}[htbp]
  \centering
  \small   
  \setlength{\tabcolsep}{5pt}
  \caption{Comparison of RB--DeepONet, POD--DeepONet, and FEONet.}
  \label{tab:cases-feonet}
  \begin{threeparttable}
  \begin{tabular}{lcccc}
    \toprule
    & \multicolumn{2}{c}{RB--DeepONet} & POD--DeepONet & FEONet \\
    \cmidrule(lr){2-3}
    Feature & Case~I & Case~II & \cite{lu2022comprehensive} & \cite{lee2025finite} \\
    \midrule
    Trunk
      & RB
      & RB
      & POD
      & FE \\[2pt]
    Input
      & $\mathbf k$
      & $\mathbf k_{\rm aug}$
      & task dep.\tnote{a}
      & task dep.\tnote{a} \\[2pt]
    Output
      & $\mathbf c_\theta\!\in\!\mathbb R^N$
      & $\mathbf c_\theta\!\in\!\mathbb R^N$
      & $\mathbf c_\theta\!\in\!\mathbb R^N$
      & $\mathbf u_\theta\!\in\!\mathbb R^{N_0}$ \\[2pt]
    Training
      & unsup.
      & unsup.
      & sup.
      & unsup. \\[2pt]
    Dirichlet BCs
      & lifting
      & modes\tnote{b}
      & lifting
      & nodal \\[2pt]
  Offline PDE solves
    & $N$ & $N$ & $N_k$ & 0 \\[2pt]
  Training cost  (per sample)
    & net($N$)\tnote{c} + RB\tnote{d}
  & net($N$) + RB
  & net($N$) + full($N_0 N$)\tnote{e}
  & net($N_0$) + FE\tnote{g} \\[2pt]%
  Online prediction cost
    & $\mathcal O(N)$
    & $\mathcal O(N+r_g +r_f)$
    & $\mathcal O(N)$
    & $\mathcal O(N_0)$ \\
  \bottomrule
\end{tabular}
\begin{tablenotes}
  \footnotesize
  \item[a] Task-dependent subset of $\{\mathbf k,f,g_D,h_N,r_R\}$.
  \item[b] Approximate via boundary modes (with projection error).
  \item[c] net($M$): cost of one forward/backward pass of the branch network 
         with $M$-dimensional output.
\item[d] RB: reduced residual evaluation in $\mathbb R^N$ (order $N^2$) with additional $\mathcal O(N r_g + N r_f)$ work in Case~II.
\item[e] full($N_0 N$): full-field reconstruction and supervised loss on 
         $N_0$ degrees of freedom.
\item[f] FE: full-order FE residual evaluation.
\end{tablenotes}
  \end{threeparttable}
\end{table}

\begin{remark}[Generality of RB--DeepONet]\label{remark:Generality}
For clarity, we develop the method in the prototypical second--order elliptic setting \eqref{eq:strong}. 
However, the algorithm only relies on three ingredients:  
(A1) a variational formulation (linear or semilinear) that is uniformly well-posed ;  
(A2) an offline reduced trial space $\mathrm{span}\{\psi_i\}_{i=1}^N\subset V$ together with the RB residual operator $\mathbf A_{\rm rb}(\mathbf k)$ obtained by testing with a compatible space; and  
(A3) an affine or EIM/DEIM surrogate enabling offline–online separation.  
Whenever (A1)–(A3) hold, the trunk fixing, label–free residual training, and the boundary/source mode treatment carry over verbatim.
\end{remark}

\section{Convergence analysis of RB--DeepONet}\label{sec:Convergence analysis of RB--DeepONet}

We now analyze the convergence of the RB--DeepONet prediction $w_{\theta}(\mathbf{k})=\Psi\,\mathbf{c}_\theta(\mathbf{k})$ to the reduced-basis Galerkin solution
$w_{\rm rb}(\mathbf{k})=\Psi\,\mathbf{c}_N(\mathbf{k})$, with the finite element mesh $\mathcal T_h$ and
the RB space $V_{\rm rb}=\mathrm{span}\{\psi_i\}_{i=1}^N\subset V_h$ fixed. 
To simplify the discussion, we focus on Case I, where the PDEs are fully parameterized. For Case II, the process is similar: we only need to replace $\mathbf{k}$ with $\mathbf{k}_{\rm aug}$ and $\mathbf{F}_{\rm rb}(\mathbf{k})$ with $\mathbf{F}_{\rm rb}(\mathbf{k}_{\rm aug})$, and the only additional modeling error is the boundary and source projections, which is quantified in Section \ref{sec:bd-src-construct}.
Note that we regard the reduced basis $\{\psi_i\}_{i=1}^N$ as predetermined by Algorithm \ref{alg:greedy} on a sufficiently fine finite element mesh $\mathcal{T}_h$ such that the standard reduced basis method already provides an approximation within the desired accuracy. Hence, the error associated with RB space construction itself is not the subject of the following analysis.

Training adopts the RB variational residual \eqref{eq: rb residual} as the foundation for residual minimization. Accordingly, the loss function \eqref{eq:forward-loss} measures the departure of the residual from zero.
Our convergence analysis therefore depends on two key factors:
\begin{enumerate}
    \item the expressive power of the branch network, which is influenced by its width and depth;
    \item the number of parameter samples $N_s$ used to calculate the loss function \eqref{eq:forward-loss}.
\end{enumerate}

This separation isolates the network approximation from the RB discretization. We split the total error as
\begin{equation}\label{eq:error-split}
    w-\widehat{w}_{n,N_s}
    \;=\; \underbrace{(w-w_{\rm rb})}_{\text{RB discretization}} \;+\; \underbrace{(w_{\rm rb}-\widehat w_{n})}_{\text{approximation}} \;+\; \underbrace{(\widehat w_{n}-\widehat{w}_{n,N_s})}_{\text{generalization}},
\end{equation}
Here, the first term $(w-w_{\rm rb})$ is the (fixed) RB discretization error that depends only on the Greedy construction and decays rapidly given the exponential decay of the Kolmogorov width for many elliptic PDE
families \cite{quarteroni2015reduced,lucia2004reduced,maday2006reduced}. The second term $(w_{\rm rb}-\widehat w_{n})$ is the network
\emph{approximation} error for a class of neural network with parameter size $n$ (e.g., width and depth), and
$(\widehat w_{n}-\widehat{w}_{n,N_s})$ is the \emph{generalization} error, i.e., the difference between the idealized neural network solution $\widehat w_{n}$ obtained from exact residual minimization and the practical solution $\widehat{w}_{n,N_s}$ obtained using a finite number $N_s$ of training samples. 
In the following, we focus on the latter two terms.

\begin{proposition}\label{prop:rb-stability}
Assume Assumption \ref{ass:wellposed} holds. Furthermore, we assume the maps
$\mathbf{k}\mapsto a(\cdot,\cdot;\mathbf{k})$ and $\mathbf{k}\mapsto \ell(\cdot;\mathbf{k})$ are continuous on $\mathcal{D}$.
Then, for a fixed RB basis $\Psi=[\psi_1,\ldots,\psi_N]\subset V_0$ as in Section~\ref{sec:rb-construction} and $\mathbf A_{\rm rb}(\mathbf{k})$, $\mathbf F_{\rm rb}(\mathbf{k})$ as in \eqref{eq:rb-operators},
there exist constants $0<\sigma_*\le\sigma^*<\infty$ and $C_F>0$, independent of $\mathbf{k}$, such that
\[
  \sigma_* \|\mathbf z\|_2^2 \;\le\; \mathbf z^\top \mathbf A_{\rm rb}(\mathbf{k})\,\mathbf z \;\le\; \sigma^* \|\mathbf z\|_2^2,
  \qquad
  \|\mathbf F_{\rm rb}(\mathbf{k})\| \;\le\; C_F,  \quad \forall\,\mathbf z\in\mathbb{R}^N, \quad\forall\,\mathbf{k}\in\mathcal{D} .
\]
Consequently, the RB system $\mathbf A_{\rm rb}(\mathbf{k})\,\mathbf c_N(\mathbf{k})=\mathbf F_{\rm rb}(\mathbf{k})$ is uniformly well posed for all
$\mathbf{k}\in\mathcal{D}$, and the coefficient map $\mathbf{k}\mapsto \mathbf c_N(\mathbf{k})$ is continuous on $\mathcal{D}$.
\end{proposition}

Recall that we work on the compact parameter set $\mathcal D\subset\mathbb R^p$, endowed with the uniform probability measure $\rho$, i.e.,
$d\rho(\mathbf k)=|\mathcal D|^{-1}\,d\mu(\mathbf k)$ where $\mu$ is the Lebesgue measure and $|\mathcal D|$ the volume of $\mathcal D$.
Before proceeding with the analysis, we introduce the spaces and norms used below.
We write $L^2(\Omega)$ for the space of square-integrable functions on $\Omega$ with norm
\[
\|v\|_{L^2(\Omega)}^2:=\int_\Omega |v(x)|^2\,dx.
\]
For a coefficient map $\mathbf c:\mathcal D\to\mathbb R^N$, endowed pointwise with the Euclidean norm $\|\cdot\|_2$ on $\mathbb R^N$, we define
\[
\|\mathbf c\|_{L^2(\mathcal D;\rho)}^2
:=\int_{\mathcal D}\|\mathbf c(\mathbf k)\|_2^2\,d\rho(\mathbf k).
\]
We write $C(\mathcal D;\mathbb R^N)$ for the Banach space of continuous maps
$g:\mathcal D\to\mathbb R^N$ endowed with the sup-norm
\[
\|g\|_{C(\mathcal D)}:=\sup_{\mathbf k\in\mathcal D}\|g(\mathbf k)\|_2.
\]
For a function-valued map $w:\mathcal D\to L^2(\Omega)$ that is Bochner measurable, we define 
\begin{equation}\label{eq: Bochner-def}
\|w\|_{L^2(\mathcal D;L^2(\Omega))}^2
:=\int_{\mathcal D}\|w(\mathbf{x};\mathbf k)\|_{L^2(\Omega)}^2\,d\rho(\mathbf k)
=\int_{\mathcal D}\int_{\Omega} |w(\mathbf{x};\mathbf k)|^2\,dx\,d\rho(\mathbf k).
\end{equation}
Let $\{\psi_i\}_{i=1}^N\subset L^2(\Omega)$ be the fixed RB basis and define the linear map
$\mathcal T:\mathbb R^N\to L^2(\Omega)$ by
$\mathcal T(\boldsymbol\alpha):=\sum_{i=1}^N \alpha_i\psi_i$.
By finite dimensionality,
\[
\|\mathcal T\|
\;\le\; \Big(\sum_{i=1}^N \|\psi_i\|_{L^2(\Omega)}^2\Big)^{1/2}
=: C_\psi <\infty.
\]
Consequently, for any $\mathbf c\in L^2(\mathcal D;\mathbb R^N)$, the map
$w(\cdot;\mathbf k):=\mathcal T(\mathbf c(\mathbf k))=\sum_{i=1}^N c_i(\mathbf k)\psi_i$ belongs to
$L^2(\mathcal D;L^2(\Omega))$ and satisfies the lifting estimate
\begin{equation}\label{eq: lifting-bound}
\|w\|_{L^2(\mathcal D;L^2(\Omega))}
\;\le\; \|\mathcal T\|\,\|\mathbf c\|_{L^2(\mathcal D;\rho)}
\;\le\; C_\psi\,\|\mathbf c\|_{L^2(\mathcal D;\rho)}.
\end{equation}
Given a function class $\mathcal F\subset \{f:\mathcal D\to\mathbb R\}$ and a probability measure $P$ on $\mathcal D$, we define
\[
\|f-g\|_{L_2(P)}\ :=\ \Big(\int_{\mathcal D}|f(\mathbf k)-g(\mathbf k)|^2\,dP(\mathbf k)\Big)^{1/2}.
\]
For a sample $S=(\mathbf k_1,\ldots,\mathbf k_{N_s})$, the corresponding empirical norm is
\[
\|f-g\|_{L_2(P_S)}\ :=\ \Big(\tfrac1{N_s}\sum_{i=1}^{N_s}|f(\mathbf k_i)-g(\mathbf k_i)|^2\Big)^{1/2}.
\]

Given any coefficient map $\mathbf{c}:\mathcal{D}\to\mathbb{R}^N$, we now define the population loss
\begin{equation}\label{eq:population-loss}
    \mathcal L^{\mathrm{pop}}(c)
:= \int_{\mathcal D} \big\|\mathbf A_{\mathrm{rb}}(\mathbf k)\,\mathbf c(\mathbf k)
       - \mathbf F_{\mathrm{rb}}(\mathbf k)\big\|_2^2 \, d\rho(\mathbf k),
\end{equation}
which represents the expected residual error with respect to the uniform distribution on $\mathcal D$. 
In practice, we approximate the population loss \eqref{eq:population-loss} by Monte--Carlo integration using
i.i.d.\ samples $\{\mathbf k_j\}_{j=1}^{N_s}\stackrel{\text{i.i.d.}}{\sim}\rho$.
This yields the empirical loss function
\begin{equation}\label{eq: loss N_k}
\begin{aligned}
    \mathcal{L}^{N_s}(\mathbf{c})
    := \frac{1}{N_s}\sum_{j=1}^{N_s}\left\|\mathbf{A}_{\rm rb}(\mathbf{k}_j)\mathbf{c}(\mathbf{k}_j)-\mathbf{F}_{\rm rb}(\mathbf{k}_j)\right\|_2^2.
\end{aligned}    
\end{equation}
Note that by the uniform spectral bounds on $\mathbf{A}_{\rm rb}(\mathbf{k})$, the preconditioned loss \eqref{eq:forward-loss} and the unweighted residual loss \eqref{eq: loss N_k} are uniformly equivalent; in particular, they share the same population minimizer.

Then, we have the following proposition.

\begin{proposition}\label{prop:1}
Under Proposition \ref{prop:rb-stability}, the unique minimizer of \eqref{eq:population-loss} over
$C(\mathcal D;\mathbb R^N)$ is the continuous RB coefficient map $ \mathbf{c}_N(\mathbf{k})=(\mathbf{A}_{\rm rb}(\mathbf{k}))^{-1}\mathbf{F}_{\rm rb}(\mathbf{k})$ for $\mathbf k\in\mathcal D$, i.e.,
\begin{equation}\label{eq:minimizer-pop}
    \mathbf{c}_N=\arg\min_{\mathbf{c}\in C(\mathcal{D};\mathbb{R}^N)}\mathcal{L}^{\rm pop}(\mathbf{c}).
\end{equation}
Moreover, for any $\mathbf c\in C(\mathcal D;\mathbb R^N)$,
\begin{equation}\label{eq:pop-strong}
    \mathcal L^{\rm pop}(\mathbf c)
    \;\ge\; \sigma_*^2\,\|\mathbf c-\mathbf c_N\|_{L^2(\mathcal D;\rho)}^{2}.
\end{equation}
\end{proposition}

\begin{proof}
Fix $\mathbf k\in\mathcal D$ and define
\[
f_{\mathbf k}(\mathbf z):=\|\mathbf{A}_{\rm rb}(\mathbf{k})\mathbf z-\mathbf{F}_{\rm rb}(\mathbf{k})\|_2^2,
\]
where $\mathbf z\in\mathbb R^N$.
Since $\mathbf{A}_{\rm rb}(\mathbf{k})$ is invertible, $f_{\mathbf k}$ is a strictly convex quadratic with unique minimizer
$(\mathbf{A}_{\rm rb}(\mathbf{k}))^{-1}\mathbf{F}_{\rm rb}(\mathbf{k})$, which proves \eqref{eq:minimizer-pop}.
For any $\mathbf c\in C(\mathcal D;\mathbb R^N)$, a pointwise identity holds:
\[
f_{\mathbf k}\big(\mathbf c(\mathbf k)\big)-f_{\mathbf k}\big(\mathbf c_N(\mathbf k)\big)
=\big\|\mathbf{A}_{\rm rb}(\mathbf{k})\big(\mathbf c(\mathbf k)-\mathbf c_N(\mathbf k)\big)\big\|_2^2.
\]
Integrating over $\mathcal D$ yields
\begin{equation}
    \begin{aligned}
        \mathcal L^{\rm pop}(\mathbf c)-\mathcal L^{\rm pop}(\mathbf c_N)
=&\big\|\mathbf{A}_{\rm rb}(\mathbf{k})\big(\mathbf c-\mathbf c_N\big)\big\|_{L^2(\mathcal D;\rho)}^{2}\\
\geq& \sigma_*^2\|\mathbf c-\mathbf c_N\|_{L^2(\mathcal D;\rho)}^2,
    \end{aligned}
\end{equation}
which proves \eqref{eq:pop-strong}.
\end{proof}

Let $\{\mathcal N_n\}_{n\in\mathbb N}\subset C(\mathcal D;\mathbb R^N)$ be a nested family of vector-valued neural-network hypothesis classes. 
For each $n$, fix a network architecture of size $n$ (e.g., increasing width/depth) and let $\Theta_n\subset\mathbb R^{P_n}$ denote its parameter space. 
Define
\[
  \mathcal N_n \;:=\; \{\, \mathbf c_\theta : \mathcal D \to \mathbb R^N \;|\; \theta\in\Theta_n \,\}.
\]
To quantify the expressive capacity of such network families, we first recall the notion of pseudo-dimension, and then impose the structural assumptions summarized in Assumption \ref{ass:nn} below.

\begin{definition}[Pseudo-dimension {\cite{anthony2009neural,bartlett2019nearly,shalev2014understanding}}]
Let $\mathcal F$ be a class of real-valued functions on a set $\mathcal X$.
The pseudo-dimension $Pdim(\mathcal F)$ is the largest integer $m$ for which there exist points
$\{x_i\}_{i=1}^m\subset\mathcal X$ and thresholds $\{y_i\}_{i=1}^m\subset\mathbb R$ such that for every
$\mathbf b=(b_1,\ldots,b_m)\in\{0,1\}^m$ there exists $f\in\mathcal F$ with
\[
\mathbf 1\{f(x_i)\ge y_i\}=b_i,\qquad i=1,\ldots,m.
\]
\end{definition}

The pseudo-dimension provides a measure of the complexity of neural networks. 
In particular, it is directly connected to uniform convergence properties and generalization error bounds (see, e.g., \cite{anthony2009neural,bartlett2019nearly}).
Then, we assume the following standard NN structural conditions.

\begin{assumption}\label{ass:nn}
The classes $\{\mathcal N_n\}_{n\in\mathbb N}$ satisfy:
\begin{enumerate}
    \item \textbf{Nestedness:} $\mathcal N_n\subset\mathcal N_{n+1}$ for all $n\in\mathbb N$.
    \item \textbf{Uniform boundedness:} There exists $M_\infty>0$ such that  
    \[
    sup_{g\in\cup_n\mathcal N_n}\|g\|_{C(\mathcal{D})}\le M_\infty.
    \]
    \item \textbf{Density:} The union $\bigcup_n\mathcal N_n$ is dense in the closed ball $B_{M_\infty}:=\{g\in C(\mathcal{D};\mathbb{R}^N):\|g\|_{\mathcal{C}(\mathcal{D})}\leq M_\infty\}$, i.e., for every $g^*\in C(\mathcal{D};\mathbb R^N)$ with $\|g^*\|_{\mathcal{C}(\mathcal{D})}\leq M_\infty$,
    \[
 \lim_{n\to\infty}\inf_{g\in\mathcal{N}_n}\|g-g^*\|_{C(\mathcal{D})}=0.
    \]
    \item \textbf{Finite pseudo-dimension:} For each $n$, $p_n:=Pdim(\mathcal N_n)<\infty$.
\end{enumerate}
\end{assumption}

The above assumption follows \cite{ko2022convergence,lee2025finite}, which provides the minimal analytic framework for establishing both approximation and generalization properties of the neural network. 
In particular, conditions (1)–(3) ensure the representational capability of the networks through nestedness, uniform boundedness, and density, while condition (4) controls their statistical complexity via the pseudo-dimension.
Together with Proposition \ref{prop:rb-stability}, these requirements guarantee that the network classes $\mathcal N_n$ enjoy the universal approximation property \cite[Theorem~2.2]{ko2022convergence}:

\begin{theorem}\label{thm:universal}
Under Proposition \ref{prop:rb-stability} and Assumption \ref{ass:nn}, we have
\begin{equation}
\lim_{n\to\infty}\ \inf_{\mathbf c_n\in\mathcal N_n}\ \|\mathbf c_n-\mathbf c_N\|_{C(\mathcal D)}\;=\;0.
\end{equation}
\end{theorem}

Given a neural-network hypothesis class $\mathcal N_n\subset C(\mathcal D;\mathbb R^N)$, we define the population risk minimizer and the associated RB approximation
\begin{equation}\label{eq:hn-pop}
    \widehat{\mathbf c}_{n}\;=\;\operatorname*{arg\,min}_{\mathbf c\in\mathcal N_n}\ \mathcal{L}^{\rm pop}(\mathbf{c}),\qquad \widehat w_{n}(\mathbf x;\mathbf k)\;=\;\sum_{i=1}^N \widehat{\mathbf c}_{n,i}(\mathbf k)\,\psi_i(\mathbf x).
\end{equation}
Here $\widehat{\mathbf{c}}_{n,i}(\mathbf{k})$ is the $i$-th component of $\widehat{\mathbf{c}}_{n}(\mathbf{k})$.
The discrete residual minimization problem on the network class \(\mathcal N_n\) and the corresponding RB--DeepONet prediction is
\begin{equation}\label{eq: emp-min}
  \widehat{\mathbf c}^{\,N_s}_{n}
  \;:=\;
  \arg\min_{\mathbf c\in\mathcal N_n}\;
  \mathcal L^{N_s}(\mathbf c),\qquad  \widehat{w}_{n,N_s}(\mathbf x;\mathbf k)
  \;:=\; \sum_{i=1}^N \widehat{\mathbf c}^{\,N_s}_{n,i}(\mathbf k)\,\psi_i(\mathbf x).
\end{equation}

\begin{remark}[On existence of minimizers]
Since $\mathcal N_n$ is not assumed to be compact in $C(D;\mathbb R^N)$,
the minima in \eqref{eq:hn-pop}–\eqref{eq: emp-min} need not be attained \emph{a priori}.
Throughout the analysis, we therefore interpret 
$\widehat{\mathbf c}_n$ and $\widehat{\mathbf c}_{n}^{N_s}$ as (possibly approximate) global minimizers,
i.e. elements of $\mathcal N_n$ whose risks are arbitrarily close to the infima
over $\mathcal N_n$. This is standard in learning-theoretic analyses and allows us
to focus on approximation and generalization errors rather than optimization errors.
\end{remark}

Our first convergence result is presented below:

\begin{theorem}\label{thm:convergence1}
Under Proposition \ref{prop:rb-stability} and Assumptions \ref{ass:nn},
\begin{equation}\label{eq:limit1}
\lim_{n\to\infty} \|\widehat{\mathbf{c}}_{n}-\mathbf{c}_{N}\|_{L^2(\mathcal D;\rho)}=0,
\qquad
\lim_{n\to\infty} \|\widehat w_{n}-w_{\rm rb}\|_{L^2(\mathcal D;L^2(\Omega))}=0,
\end{equation}
where $w_{\rm rb}(\cdot;\mathbf k)=\sum_{i=1}^N \mathbf c_{N,i}(\mathbf k)\,\psi_i(\cdot)$.
\end{theorem}

\begin{proof}
By Proposition~\ref{prop:1}, for any $\mathbf c\in C(\mathcal D;\mathbb R^N)$,
\begin{equation}\label{eq:quad-geometry}
\mathcal L^{\rm pop}(\mathbf c)
\ge \sigma_*^2\,\|\mathbf c-\mathbf c_N\|_{L^2(\mathcal D;\rho)}^2.
\end{equation}
Evaluating at $\widehat{\mathbf c}_n$, then using \eqref{eq:hn-pop} and Proposition \ref{prop:rb-stability},
\begin{equation}\label{eq:left-bound}
\begin{aligned}
    \|\widehat{\mathbf c}_{n}-\mathbf c_N\|_{L^2(\mathcal D;\rho)}^2
\ \le&\ \sigma_*^{-2}\,\mathcal L^{\rm pop}(\widehat{\mathbf c}_n)
\ =\ \sigma_*^{-2}\,\inf_{\mathbf c\in\mathcal N_n}\mathcal L^{\rm pop}(\mathbf c)\\
=&\ \sigma_*^{-2}\,\inf_{\mathbf c\in\mathcal N_n}\|\mathbf{A}_{\rm rb}(\mathbf{k})\left(\mathbf c-\mathbf c_N\right)\|_{L^2(\mathcal D;\rho)}^2\\
\le&\ \Big(\frac{\sigma^*}{\sigma_*}\Big)^2
\inf_{\mathbf c\in\mathcal N_n}\|\mathbf c-\mathbf c_N\|_{L^2(\mathcal D;\rho)}^2.
\end{aligned}
\end{equation}
By Theorem \ref{thm:universal}, we obtain
$\|\widehat{\mathbf c}_{n}-\mathbf c_N\|_{L^2(\mathcal D;\rho)}\to0$ as $n\to\infty$.
Then the function error is derived by the boundedness of the linear map $\mathcal T$, which proves \eqref{eq:limit1}.
\end{proof}

In order to control the \emph{generalization} error between the empirical and population loss, i.e., \eqref{eq:population-loss} and \eqref{eq: loss N_k}, we first introduce the notion of Rademacher complexity.

\begin{definition}[Rademacher complexity {\cite{bartlett2002rademacher}}]
Let $\mathcal D\subset\mathbb R^p$ be compact with probability measure $\rho$.
Given a sample $S=(\mathbf k_1,\ldots,\mathbf k_{N_s})\stackrel{\text{i.i.d.}}{\sim}\rho$ and i.i.d.\ Rademacher variables
$\varepsilon_1,\ldots,\varepsilon_{N_s}$ with
$\mathbb P(\varepsilon_i=\pm1)=\tfrac12$, consider the real-valued loss class
\[
g_{\mathbf c}(\mathbf k):=\big\|\mathbf{A}_{\rm rb}(\mathbf{k})\,\mathbf c(\mathbf k)-\mathbf{F}_{\rm rb}(\mathbf{k})\big\|_2^2,
\qquad 
\mathcal G_n:=\{\,g_{\mathbf c}:\ \mathbf c\in\mathcal N_n\,\}.
\]
The \emph{empirical} Rademacher complexity of $\mathcal G_n$ on the fixed sample $S$ is
\[
\widehat R_{S}(\mathcal G_n)
:=\mathbb E_{\varepsilon}\!\left[\ \sup_{g\in\mathcal G_n}\frac1{N_s}\sum_{i=1}^{N_s}\varepsilon_i\,g(\mathbf k_i)\ \right].
\]
The \emph{expected} Rademacher complexity with sample size $N_s$ is
\[
R_{N_s}(\mathcal G_n)
:=\mathbb E_{S}\big[\widehat R_{S}(\mathcal G_n)\big]
=\mathbb E_{S,\varepsilon}\!\left[\ \sup_{g\in\mathcal G_n}\frac1{N_s}\sum_{i=1}^{N_s}\varepsilon_i\,g(\mathbf k_i)\ \right].
\]
\end{definition}
In this setting, the empirical and population losses can be written compactly as
\[
\mathcal{L}^{N_s}(\mathbf{c})=\frac{1}{N_s}\sum_{i=1}^{N_s} g_{\mathbf{c}}(\mathbf{k}_i),
\qquad 
\mathcal{L}^{pop}(\mathbf{c})=\int_{\mathcal{D}} g_{\mathbf{c}}(\mathbf{k})\,d\rho(\mathbf{k}).
\]
By Proposition \ref{prop:rb-stability} and Assumption \ref{ass:nn}, we have
\begin{equation}\label{eq:uniform bound g}
    0\ \le\ g_{\mathbf c}(\mathbf k)\ \le\ \big(\sigma^* M_\infty + C_F\big)^2\;=:\;B,
\end{equation}
for all $g_{\mathbf{c}}\in\mathcal G_n$, $ \mathbf{k}\in\mathcal D$ and $\mathbf c\in\mathcal N_n$.
Hence, $\mathcal G_n\subset[0,B]^{\mathcal D}$.
We have derived the following uniform law:

\begin{lemma}[Uniform convergence {\cite[Theorem 4.10]{wainwright2019high}, \cite[Theorem 26.5]{shalev2014understanding}}]\label{lemma: uniform-law}
For any $\delta\in(0,1)$, with probability at least $1-\delta$ over the draw of
$\{\mathbf k_i\}_{i=1}^{N_s}\stackrel{\text{i.i.d.}}{\sim}\rho$,
\begin{equation}\label{eq: uniform-law}
\sup_{\mathbf c\in\mathcal N_n}\big|\ \mathcal L^{N_s}(\mathbf c)-\mathcal L^{\rm pop}(\mathbf c)\ \big|
\ \le\ 2\,R_{N_s}(\mathcal G_n)\ +\ B\,\sqrt{\frac{2\log(1/\delta)}{N_s}}.
\end{equation}
\end{lemma}

We next control $R_{N_s}(\mathcal G_n)$ on the RHS of \eqref{eq: uniform-law} in terms of the pseudo-dimension
$p_n$.

\begin{definition}[Covering number and $\epsilon$-net {\cite{van1996weak}}]\label{def: covering-number}
Let $(\mathcal{V},\|\cdot\|)$ be a normed linear space and let $\mathcal F\subset \mathcal{V}$.
A finite set $E\subset \mathcal F$ is called an $\epsilon$-net of $\mathcal F$ with respect to $\|\cdot\|$ if
for every $f\in\mathcal F$ there exists $e\in E$ such that $\|f-e\|\le \epsilon$.
The $\epsilon$-covering number of $\mathcal F$ is
\[
\mathcal N(\mathcal F,\|\cdot\|,\epsilon)
\ :=\ \min\big\{\,|E|:\ E\subset\mathcal F \text{ is an $\epsilon$-net of } \mathcal F \text{ w.r.t.\ }\|\cdot\|\,\big\},
\]
with the convention $\mathcal N(\mathcal F,\|\cdot\|,\epsilon)=+\infty$ if no finite $\epsilon$-net exists.
\end{definition}

\begin{lemma}\label{lem: Rademacher-rate}
Fix $n\in\mathbb N$. Under Proposition \ref{prop:rb-stability} and Assumption \ref{ass:nn}, there exists a universal
constant $C>0$, which is independent of $n,N_s$, such that
\begin{equation}\label{eq: RC-rate}
R_{N_s}(\mathcal G_n)\ \le\ C\,B\,\sqrt{\frac{p_n}{N_s}}.
\end{equation}
In particular, for each fixed $n$, $R_{N_s}(\mathcal G_n)\to0$ as $N_s\to\infty$.
\end{lemma}

\begin{proof}
For $\mathbf c_1,\mathbf c_2\in\mathcal N_n$ and any $\mathbf k\in\mathcal D$, set
$h_{\mathbf c}(\mathbf k):=\mathbf{A}_{\rm rb}(\mathbf{k})\mathbf c(\mathbf k)-\mathbf{F}_{\rm rb}(\mathbf{k})$ and note
\[
\big|g_{\mathbf c_1}(\mathbf k)-g_{\mathbf c_2}(\mathbf k)\big|
=\big|\|h_{\mathbf c_1}(\mathbf k)\|_2^2-\|h_{\mathbf c_2}(\mathbf k)\|_2^2\big|
\le \big(\|h_{\mathbf c_1}(\mathbf k)\|_2+\|h_{\mathbf c_2}(\mathbf k)\|_2\big)\,\|h_{\mathbf c_1}(\mathbf k)-h_{\mathbf c_2}(\mathbf k)\|_2.
\]
Since we have the uniform range bound \eqref{eq:uniform bound g}, we can get the Lipschitz transfer
\[
\|g_{\mathbf c_1}-g_{\mathbf c_2}\|_{L_2(P)}
\ \le\ L\,\|\mathbf c_1-\mathbf c_2\|_{L_2(P)},
\qquad L:=2(\sigma^* M_\infty+C_F)\,\sigma^* .
\]
Hence, by the standard covering-number transfer,
\begin{equation}\label{eq: covering number eq1}
    \mathcal N\!\big(\mathcal G_n,\|\cdot\|_{L_2(P)},\epsilon\big)
\ \le\ \mathcal N\!\big(\mathcal N_n,\|\cdot\|_{L_2(P)},\epsilon/L\big),
\qquad \forall\,\epsilon>0,\ \forall\,P.
\end{equation}

For vector-valued networks with finite pseudo-dimension $p_n$, standard results in \cite[Theorem~18.4]{anthony2009neural} imply that for any probability
measure $P$ on $\mathcal D$ and any $\varepsilon\in(0,M_\infty]$,
\begin{equation}\label{eq: cover-Nn}
\mathcal N\!\big(\mathcal N_n,\|\cdot\|_{L_2(P)},\varepsilon\big)
\ \le\ \Big(\frac{C_1\,M_\infty}{\varepsilon}\Big)^{C_2\,p_n},
\end{equation}
for universal constants $C_1,C_2>0$. 
Combining this with \eqref{eq: covering number eq1} and absorbing constants yields, for every
$\epsilon\in(0,B]$,
\begin{equation}\label{eq: cover-Gn}
\mathcal N\!\big(\mathcal G_n,\|\cdot\|_{L_2(P)},\epsilon\big)
\ \le\ \Big(\frac{C_3\,B}{\epsilon}\Big)^{C_2\,p_n},
\end{equation}
where $C_3>0$ is a universal constant independent of $n$ and $N_s$.


By Dudley’s entropy integral bound for the empirical Rademacher complexity
\cite[Theorem~27.4]{shalev2014understanding}, for any $\alpha\in(0,B]$,
\[
\widehat R_S(\mathcal G_n)
\ \le\ \alpha\ +\ \frac{6}{\sqrt{N_s}}
\int_{\alpha}^{B}\sqrt{\log \mathcal N\!\big(\mathcal G_n,\|\cdot\|_{L_2(P_{N_s})},\epsilon\big)}\,d\epsilon .
\]
Taking expectation in the sample $S$ and using that $\sup_{P}$ dominates the empirical metric gives
\[
R_{N_s}(\mathcal G_n)\ \le\ \alpha\ +\ \frac{6}{\sqrt{N_s}}
\int_{\alpha}^{B}\sqrt{\,\sup_{P}\log \mathcal N\!\big(\mathcal G_n,\|\cdot\|_{L_2(P)},\epsilon\big)\,}\ d\epsilon .
\]
Inserting \eqref{eq: cover-Gn} and choosing $\alpha=B/\sqrt{N_s}$, we obtain
\begin{equation}
    \begin{aligned}
        R_{N_s}(\mathcal G_n)
\ \le&\ \alpha\ +\ \frac{6\sqrt{C_2\,p_n}}{\sqrt{N_s}}
\int_{\alpha}^{B}\sqrt{\log\!\Big(\frac{C_3\,B}{\epsilon}\Big)}\ d\epsilon\\
\le&\ C\,B\,\sqrt{\frac{p_n}{N_s}},
    \end{aligned}
\end{equation}
for a universal constant $C>0$. This proves \eqref{eq: RC-rate}.
\end{proof}

Combining Lemma~\ref{lemma: uniform-law} and Lemma~\ref{lem: Rademacher-rate}, we obtain, for any fixed $n$ and any $\delta\in(0,1)$, with probability at least $1-\delta$,
\[
\sup_{\mathbf c\in\mathcal N_n}\big|\ \mathcal L^{N_s}(\mathbf c)-\mathcal L^{\rm pop}(\mathbf c)\ \big|
\ \le\ C\,B\,\sqrt{\frac{p_n}{N_s}}\ +\ B\,\sqrt{\frac{2\log(1/\delta)}{N_s}},
\]
which shows uniform convergence of the empirical loss to the population loss at the rate
$O\!\left(\sqrt{p_n/N_s}\right)$ for each fixed network size $n$.

\begin{theorem}\label{thm: convergence2}
Suppose Assumption \ref{ass:nn} holds. Then for any fixed \(n\in\mathbb N\), we have
\(\lim_{N_s\to\infty}R_{N_s}(\mathcal G_n)=0\).
Moreover, with probability \(1\) over i.i.d.\ samples \(\{\mathbf k_j\}_{j=1}^{N_s}\),
\begin{equation}\label{eq: limit2}
  \lim_{N_s\to\infty}\big\|\widehat{\mathbf c}^{\,N_s}_{n}-\widehat{\mathbf c}_{n}\big\|_{L^2(\mathcal D;\rho)}=0,
  \qquad
  \lim_{N_s\to\infty}\big\|\widehat{w}_{n,N_s}-\widehat w_{n}\big\|_{L^2(\mathcal D;L^2(\Omega))}=0.
\end{equation}
\end{theorem}

\begin{proof}
Since \(\widehat{\mathbf c}^{\,N_s}_n\) minimizes \(\mathcal L^{N_s}\) over \(\mathcal N_n\),
\[
\mathcal L^{N_s}(\widehat{\mathbf c}^{\,N_s}_n)\ \le\ \mathcal L^{N_s}(\widehat{\mathbf c}_n).
\]
Hence
\[
\mathcal L^{\rm pop}(\widehat{\mathbf c}^{\,N_s}_n)-\mathcal L^{\rm pop}(\widehat{\mathbf c}_n)
\ \le\
\big(\mathcal L^{\rm pop}-\mathcal L^{N_s}\big)(\widehat{\mathbf c}^{\,N_s}_n)
\ -\
\big(\mathcal L^{\rm pop}-\mathcal L^{N_s}\big)(\widehat{\mathbf c}_n).
\]
Taking the supremum over \(\mathcal N_n\) on the right-hand side and applying Lemma~\ref{lemma: uniform-law} yields, with probability at least \(1-\delta\),
\begin{equation}\label{eq: excess-risk}
\mathcal L^{\rm pop}(\widehat{\mathbf c}^{\,N_s}_n)-\mathcal L^{\rm pop}(\widehat{\mathbf c}_n)
\ \le\ 2\sup_{\mathbf c\in\mathcal N_n}\big|\mathcal L^{N_s}(\mathbf c)-\mathcal L^{\rm pop}(\mathbf c)\big|
\ \le\ 4\,R_{N_s}(\mathcal G_n)\ +\ 2B\,\sqrt{\tfrac{2\log(1/\delta)}{N_s}},
\end{equation}
Letting $N_s\to\infty$, Lemma \ref{lem: Rademacher-rate} shows that the right-hand side of \eqref{eq: excess-risk} converges to $0$ almost surely.

Proposition \ref{prop:rb-stability} implies 
\begin{equation}\label{eq: strong-convex}
  \sigma_*^2\|\mathbf c-\widehat{\mathbf c}_n\|_{L^2(\mathcal D;\rho)}^2
  \ \le\
  \mathcal L^{\rm pop}(\mathbf c)-\mathcal L^{\rm pop}(\widehat{\mathbf c}_n)
  \ \le\
  (\sigma^*)^2\|\mathbf c-\widehat{\mathbf c}_n\|_{L^2(\mathcal D;\rho)}^2,
  \qquad \forall\,\mathbf c\in\mathcal N_n.
\end{equation}
Combining \eqref{eq: excess-risk} and \eqref{eq: strong-convex} with \(\mathbf c=\widehat{\mathbf c}^{\,N_s}_n\) and $\delta=N_s^{-1/2}$ gives, with probability one,
\begin{equation}\label{eq: lim-c}
  \lim_{N_s\to\infty}\big\|\widehat{\mathbf c}^{\,N_s}_n-\widehat{\mathbf c}_n\big\|_{L^2(\mathcal D;\rho)}=0.
\end{equation}
Since \(\{\psi_i\}_{i=1}^N\subset L^2(\Omega)\) are fixed, the lifting bound \eqref{eq: lifting-bound} gives
\begin{equation}\label{eq: lim-u}
  \big\|\widehat{w}_{n,N_s}-\widehat u_n\big\|_{L^2(\mathcal D;L^2(\Omega))}
  \ \le\ C_\psi\,\big\|\widehat{\mathbf c}^{\,N_s}_n-\widehat{\mathbf c}_n\big\|_{L^2(\mathcal D;\rho)},
\end{equation}
which yields the second limit in \eqref{eq: limit2}.
\end{proof}

Based on Theorems \ref{thm:convergence1} and \ref{thm: convergence2} and the uniform equivalence of \eqref{eq:forward-loss} and \eqref{eq: loss N_k}, now we can address the convergence of the predicted solution by RB--DeepONet to the reduced finite element solution.

\begin{theorem}[Convergence of RB--DeepONet to the RB solution]\label{thm: final}
Suppose Assumption \ref{ass:nn} holds. For the fixed RB basis
\(\{\psi_i\}_{i=1}^N\subset L^2(\Omega)\), we have with probability \(1\) over i.i.d.\ samples \(\{\mathbf k_j\}_{j=1}^{N_s}\) that
\begin{equation}\label{eq: limit3}
  \lim_{n\to\infty}\,\lim_{N_s\to\infty}\;
  \big\|w_\theta-w_{\rm rb}\big\|_{L^2(\mathcal D;L^2(\Omega))}\;=\;0 .
\end{equation}
\end{theorem}

Using \eqref{eq: limit3}, we are able to show that our algorithm can be sufficiently close to the approximate
solution computed by the proposed scheme, as the index $n$ for the neural network architecture and the number of input samples $N_k$ go to infinity.

\section{Numerical experiments}\label{sec:Numerical Results}

This section evaluates RB--DeepONet on representative parametric PDE benchmarks and compares it with POD--DeepONet \cite{lu2022comprehensive} and FEONet \cite{lee2025finite} in terms of accuracy, efficiency, and scalability. 
For a fair comparison, we fix the number of trunk basis functions \(N\) in both RB--DeepONet and POD--DeepONet. Specifically, we first construct the POD trunk used in POD--DeepONet with tolerance \(\epsilon_{\rm POD}=10^{-7}\). This choice sets the reduced dimension \(N\). A Greedy RB trunk of the same dimension is then built from the same snapshot set and used in RB--DeepONet, so that RB--DeepONet and POD--DeepONet are compared at equal trunk size.
We follow the supervised training strategy of \cite{lu2022comprehensive} and train POD--DeepONet by minimizing the empirical discrepancy between the network prediction and the reference FEM solution, as in \eqref{eq:pod deeponet loss}. In contrast, RB--DeepONet and FEONet are trained in an unsupervised fashion using the residual-based losses \eqref{eq:forward-loss} and \eqref{eq: feonet loss}, respectively.

\subsection{Example 1: Fully-parameterized steady heat conduction}
\label{sec:ex1}
First, we consider a fully parameterized steady heat conduction model with piecewise-constant conductivity and a
prescribed heat flux at the bottom. The domain is the square
$\Omega=(-0.5,0.5)^2$ with boundary split into three parts as in Fig.~\ref{fig:1}: the bottom $\Gamma_{\rm base}=(-0.5,0.5)\times\{-0.5\}$, the top $\Gamma_{\rm top}=(-0.5,0.5)\times\{0.5\}$ and the sides $\Gamma_{\rm side}=\{\pm0.5\}\times(-0.5,0.5)$.
An inclusion $\Omega_0=\{\boldsymbol x:\|\boldsymbol x\|_2\le r_0\}$ of radius $r_0=0.2$
is embedded and we denote $\Omega_1:=\Omega\backslash\Omega_0$.
Consider the thermal conductivity $\kappa$ to be piecewise constant,
\begin{equation*}
    \kappa(\mathbf{x}; k_1) =
\begin{cases}
k_1, & \mathbf{x} \in \Omega_0, \\
1, & \mathbf{x} \in \Omega_1,
\end{cases}
\end{equation*}
so that $k_1\geq0$ controls how conductive the inclusion is relative to the background.
A constant heat flux $k_2$ is prescribed on the bottom boundary.
We collect the two parameters in 
\begin{equation*}
    \mathbf{k}:=(k_1,k_2)\in\mathcal{D}:=[0.1,10]\times[-1,1].
\end{equation*}
For each $\boldsymbol k$, the temperature $u(\cdot;\boldsymbol k)$ solves
\begin{equation}\label{eq:strong-ex1}
\begin{cases}
\nabla\!\cdot\!\big(\kappa(\cdot;k_1)\nabla u(\cdot;\boldsymbol k)\big)=0 & \text{in }\Omega,\\
u(\cdot;\boldsymbol k)=0 & \text{on }\Gamma_{\rm top},\\
\kappa(\cdot;k_1)\nabla u(\cdot;\boldsymbol k)\!\cdot\!\boldsymbol n=0 & \text{on }\Gamma_{\rm side},\\
\kappa(\cdot;k_1)\nabla u(\cdot;\boldsymbol k)\!\cdot\!\boldsymbol n=k_2 & \text{on }\Gamma_{\rm base}.
\end{cases}
\end{equation}
Here, $\boldsymbol{n}$ is outward pointing unit normal on $\partial\Omega$.
The corresponding weak form reads: for each $\mathbf{k}\in\mathcal{D}$, find $u(\mathbf{k}):=u(\mathbf{x};\mathbf{k})\in V:=\{v\in H^1(\Omega):v\big|_{\Gamma_{\rm top}}=0\}$ such that
\begin{equation}\label{eq:weak-hc}
	a(u(\mathbf{k}),v;\mathbf{k})=\ell(v;\mathbf{k}),\quad \forall v\in V,
\end{equation}
where
\begin{equation}
	\begin{aligned}
		a(u,v;\mathbf{k}):=&\int_{\Omega}\kappa(\mathbf{x}; k_1)\nabla u\cdot\nabla v,\qquad
		\ell(v;\mathbf{k}):=&k_2\int_{\Gamma_{\rm base}}v.
	\end{aligned}
\end{equation}
Since $\kappa(\mathbf{x}; k_1)\geq0.1$ for all admissible $k_1$, the bilinear form is coercive and the Lax-Milgram theorem implies existence and uniqueness of $u(\mathbf{k})\in V$ for any $\mathbf{k}\in\mathcal{D}$.

\begin{figure}[htbp]
	\centering
	\subfigure{\label{cell}
		\includegraphics[width = .47\textwidth,trim={6.3cm 0cm 6.5cm 0cm},clip]{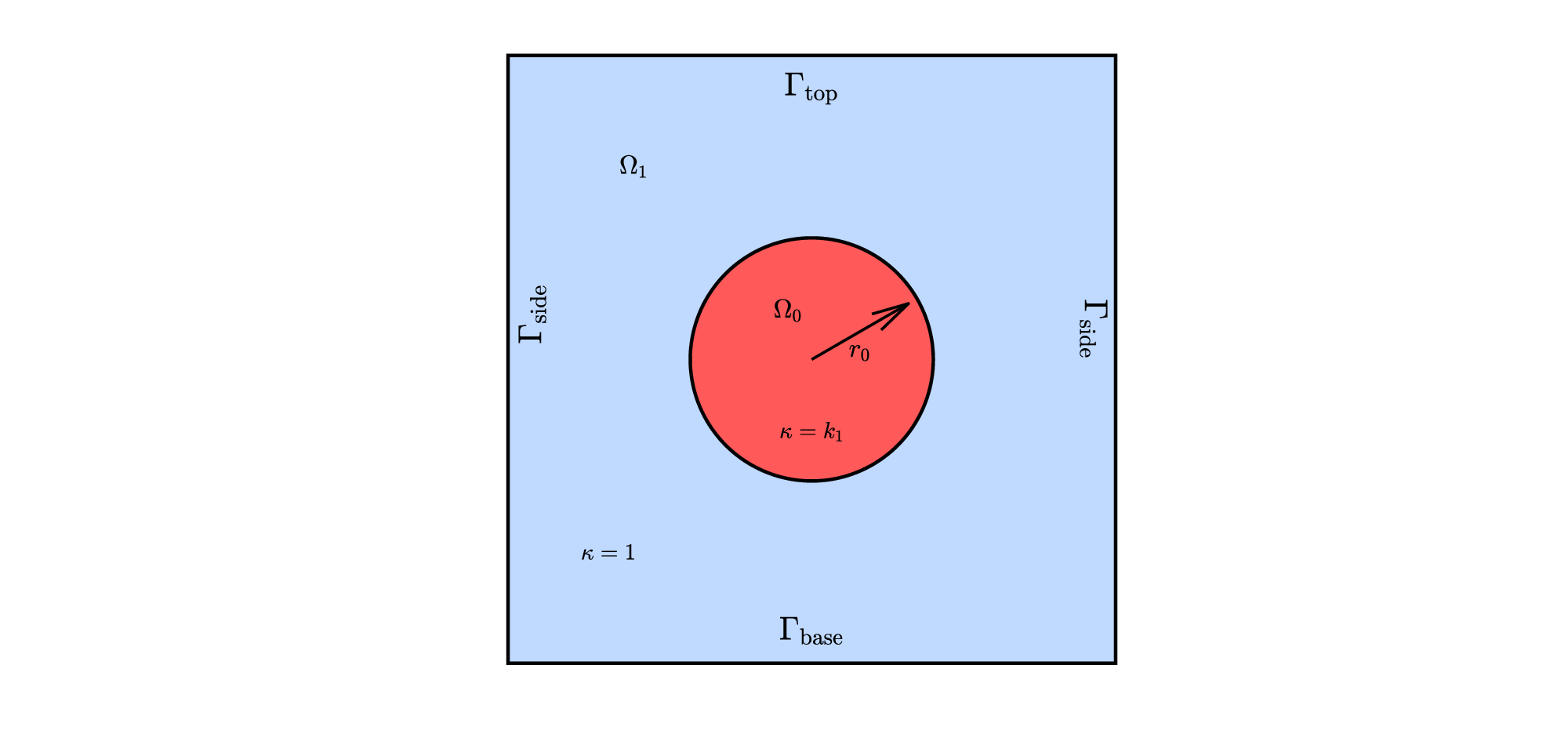}
	}%
	\subfigure{\label{mesh}
		\includegraphics[width = .47\textwidth,trim={0cm 0cm 0cm 0cm},clip]{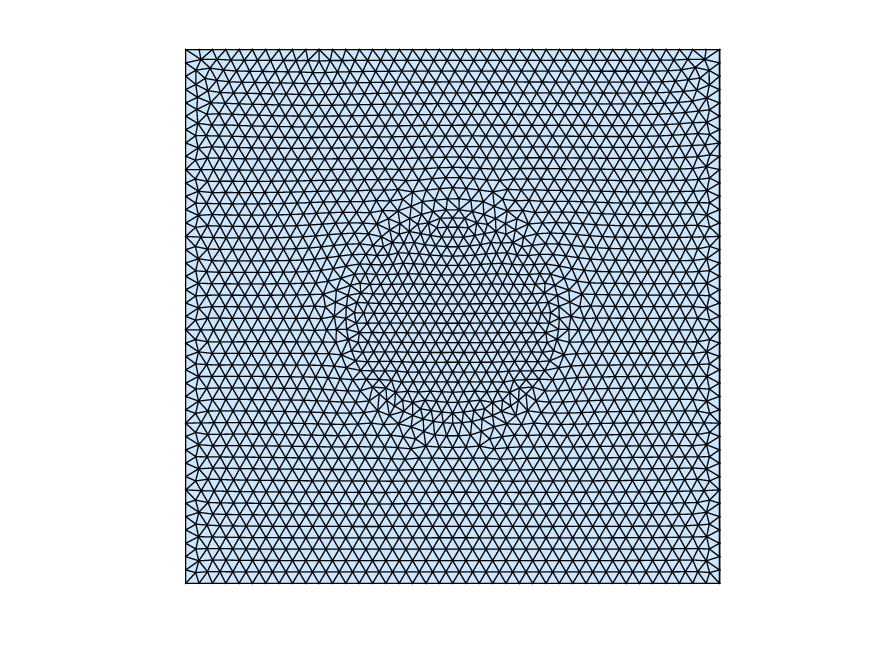}
	}%
	\centering
	\caption{Geometrical set-up (left) of the heat conductivity problem and mesh on $\Omega$ (right).}\label{fig:1}
\end{figure}

We discretize \eqref{eq:weak-hc} with continuous, piecewise linear finite elements on the conforming triangular mesh in \ref{mesh}, which has 4052 elements and $N_h=2107$ vertices.
Eliminating the homogeneous Dirichlet dofs on $\Gamma_{\rm top}$ gives
\[
\mathbf A_{II}(\boldsymbol k)\,\mathbf u_I(\boldsymbol k)=\mathbf F_I(\boldsymbol k),\qquad
\mathbf A_{II}(\boldsymbol k)=\mathbf A_1+k_1 \mathbf A_0,\quad \mathbf F_I(\boldsymbol k)=k_2 \mathbf F.
\]
Here $\mathbf A_0,\mathbf A_1$ are the stiffness contributions from $\Omega_0,\Omega_1$
and $\mathbf F$ the Neumann load on $\Gamma_{\rm base}$. They are defined on the free nodes with $N_0=2066$ dofs, i.e., $\mathbf A_0,\mathbf A_1\in\mathbb{R}^{N_0\times N_0}$, $\mathbf F\in\mathbb{R}^{N_0}$. The FE solution $\mathbf u_I(\boldsymbol k)\in\mathbb{R}^{N_0}$
serves only as a reference for RB basis construction and for evaluation.
We equip $\mathbb R^{N_0}$ with the energy inner product at $\boldsymbol k_\star=(1,1)$, i.e.,
$(\boldsymbol u_I,\boldsymbol v_I)_V=\boldsymbol u_I^\top \mathbf A_{II}(\boldsymbol k_\star)\boldsymbol v_I$.

From $N_k = 2100$ i.i.d.\ samples in $\mathcal D$, we first construct a POD
basis with tolerance $\epsilon_{\rm POD}=10^{-7}$, which yields a reduced
dimension $N=3$ and a trunk
$\Psi^{\rm POD}=[\psi^{\rm POD}_1,\psi^{\rm POD}_2,\psi^{\rm POD}_3]
\in\mathbb R^{N_0\times N}$ for POD--DeepONet.
Using the same parameter samples, we then run Greedy
algorithm to generate an RB basis of the same size,
$\Psi^{\rm G}=[\psi^{\rm G}_1,\psi^{\rm G}_2,\psi^{\rm G}_3]
\in\mathbb R^{N_0\times N}$, which is fixed as the trunk in
RB--DeepONet. Thus, POD-- and RB--DeepONet are compared at equal reduced
dimension $N$.

Given the Greedy trunk $\Psi^{\rm G}$, the intrusive RB--Galerkin system
(cf.\ \eqref{eq:rb-system}) reads: for each
$\mathbf k\in\mathcal D$, find $\mathbf c_{\rm rb}(\mathbf k)\in\mathbb R^N$
such that
\[
  \mathbf A_{\rm rb}(\mathbf k)\,\mathbf c_{\rm rb}(\mathbf k)
  = \mathbf F_{\rm rb}(\mathbf k),
\]
with $\mathbf A_{\rm rb}(\mathbf k)$ and $\mathbf F_{\rm rb}(\mathbf k)$
assembled from $\Psi^{\rm G}$ as in \eqref{eq:rb-operators}.
RB--DeepONet learns the coefficient map
$\mathbf k \mapsto \mathbf c_{\rm rb}(\mathbf k)$ and predicts
\[
  u_\theta(\mathbf k) := \Psi^{\rm G}\,\mathbf c_\theta(\mathbf k).
\]
POD--DeepONet uses the same branch architecture but replaces
$\Psi^{\rm G}$ by the POD trunk $\Psi^{\rm POD}$.
FEONet, by contrast, keeps the full FE basis
$\{\phi_j\}_{j=1}^{N_0}$ (here $N_0=2066$ piecewise linear nodal basis
functions) as trunk, and the network directly outputs
$\mathbf u_\theta(\mathbf k)\in\mathbb R^{N_0}$, the nodal vector of the
predicted FE solution.

Here are the details of our architecture and training setup:
In all three methods, we employ a fully connected multilayer perceptron (MLP) as the branch network, consisting of four hidden layers, each with a width of 256 and GELU activation functions. The input vector $\mathbf{k} \in \mathbb{R}^{2}$ is standardized using the mean and variance computed from the
training parameter samples, and the same transformation is applied during inference. Network weights are initialized using the Xavier scheme to ensure stable gradients. Training is performed for up to 2000 epochs with mini-batches of size 64 using the AdamW optimizer (initial learning rate $5\times10^{-4}$, weight decay $10^{-6}$). A ReduceLROnPlateau scheduler monitors the validation loss and halves the learning rate after 20 consecutive plateaus. The training set comprises 2000 randomly sampled parameters from $\mathcal{D}$, while validation uses 200 held-out parameters. Early stopping with a patience of 200 epochs is applied to prevent overfitting. 

Figure~\ref{fig:hc:loss} shows the training/validation losses.
All three networks reach a small residual level.
Figure~\ref{fig:hc:viz} displays a representative test parameter
$\mathbf k=(10.0,0.5)$.
The FEM reference (top left) shows a smooth vertical gradient with a
localized “thermal bubble’’ inside the high-conductivity inclusion.
FEONet, POD--DeepONet, and RB--DeepONet all reproduce this structure and
are visually indistinguishable at the plotting scale.

For a test parameter $\mathbf{k}$, let $\mathbf u_{I}(\mathbf{k})$ be the FE solution vector restricted to the free nodes, and $\mathbf e_{\rm nn}(\mathbf{k}):=\mathbf u_{\theta}(\mathbf k)-\mathbf u_{I}(\mathbf{k})$.
We report three relative error measures on the free block:
\begin{equation*}
    \begin{aligned}
        \mathrm{rel}\text{-}L^2(\mathbf k):=&\frac{\left(\mathbf e_{\rm nn}(\mathbf{k})^\top \mathbf M_{II}(\mathbf{k}_\star)\mathbf e_{\rm nn}(\mathbf{k})\right)^{1/2}}{\left(\mathbf u_{I}(\mathbf{k})^\top \mathbf M_{II}(\mathbf{k}_\star)\mathbf u_{I}(\mathbf{k})\right)^{1/2}},\\
        \mathrm{rel}\text{-}\mathrm{energy}(\mathbf k):=&\frac{\left(\mathbf e_{\rm nn}(\mathbf{k})^\top \mathbf A_{II}(\mathbf{k}_\star)\mathbf e_{\rm nn}(\mathbf{k})\right)^{1/2}}{\left(\mathbf u_{I}(\mathbf{k})^\top \mathbf A_{II}(\mathbf{k}_\star)\mathbf u_{I}(\mathbf{k})\right)^{1/2}},\\
        \mathrm{rel}\text{-}\mathrm{residual}(\mathbf k):=&\frac{\|\mathbf A_{\rm rb}(\mathbf k_\star)^{-1/2}\mathbf r(\mathbf{k})\|_2}{\|\mathbf A_{\rm rb}(\mathbf k_\star)^{-1/2}\mathbf F_{\rm rb}(\mathbf{k})\|_2}.\\
    \end{aligned}
\end{equation*}
Here, $\mathbf M_{II}$ is the mass matrix on $V_0$.
Table~\ref{tab:hc:new-agg-compact} reports the mean and 95th percentile of these errors
over $1000$ i.i.d.\ test parameters.
With only $N=3$ modes, the RB--Galerkin reference already achieves
mean rel-$L^2$ of order $10^{-7}$ and mean rel-energy of order $10^{-6}$,
which confirms that the Greedy basis $\Psi^{\rm G}$ is globally accurate.
Among the learned surrogates, POD--DeepONet and RB--DeepONet deliver very
similar accuracy.
Both attain mean rel-$L^2$ below $5\times 10^{-3}$ and 95th percentiles
below $6.1\times 10^{-3}$, so the reconstructed fields differ from the FE
solutions by well under $1\%$ on average.
POD--DeepONet has a slightly smaller mean errors than RB--DeepONet,
whereas RB--DeepONet exhibits smaller 95th-percentile errors, indicating a more uniform control of the RB residual
across the parameter space.
FEONet attains the smallest rel-$L^2$ among the three networks, but its
rel-energy and rel-residual are an order of magnitude larger, with
95th-percentile values around $3\times 10^{-2}$.
This reflects the fact that FEONet operates in the full FE space and
optimizes an unpreconditioned variational residual, which is sensitive to
the conditioning of the stiffness matrix.
Moreover, FEONet uses an output dimension equal to the number of free FE
degrees of freedom ($N_0\approx 2000$), whereas POD--DeepONet and
RB--DeepONet rely on only $N=3$ trunk modes.
Thus, RB--DeepONet and POD--DeepONet achieve FEONet-level accuracy in
rel-$L^2$ while using a trunk that is almost three orders of magnitude
smaller and, in the case of RB--DeepONet, preserving RB-type control in
energy and residual norms.

Overall, this example shows that RB--DeepONet and POD--DeepONet can match
the accuracy of FEONet up to a small factor while using an extremely
low-dimensional trunk, and RB--DeepONet in particular inherits the
RB--Galerkin error level and lightweight inference of a reduced model.
This highlights the advantage of operating entirely in a certified RB
space instead of predicting a full-order FE vector.

\begin{table}[htbp]
\centering
\small
\setlength{\tabcolsep}{6pt}
\begin{tabular}{lccc}
\toprule
\textbf{Trunk} 
& \textbf{rel-$L^2$} (mean / $p95$)
& \textbf{rel-energy} (mean / $p95$)
& \textbf{rel-residual} (mean / $p95$)
\\
\midrule
RB--DeepONet
& $\mathbf{4.99{\times}10^{-3}}$ / $\mathbf{4.66{\times}10^{-3}}$
& $\mathbf{4.99{\times}10^{-3}}$ / $\mathbf{4.56{\times}10^{-3}}$
& $\mathbf{5.14{\times}10^{-3}}$ / $\mathbf{4.90{\times}10^{-3}}$
\\[2pt]
POD--DeepONet
& $3.98{\times}10^{-3}$ / $6.05{\times}10^{-3}$
& $4.08{\times}10^{-3}$ / $6.30{\times}10^{-3}$
& $4.97{\times}10^{-3}$ / $7.50{\times}10^{-3}$
\\[2pt]
FEONet
& $2.67{\times}10^{-3}$ / $5.17{\times}10^{-3}$
& $1.26{\times}10^{-2}$ / $2.89{\times}10^{-2}$
& $1.44{\times}10^{-2}$ / $3.12{\times}10^{-2}$
\\[2pt]
\midrule
RB--Galerkin
& $2.11{\times}10^{-7}$ / $4.23{\times}10^{-7}$
& $2.77{\times}10^{-6}$ / $5.70{\times}10^{-6}$
& $7.57{\times}10^{-16}$ / $1.54{\times}10^{-15}$
\\
\bottomrule
\end{tabular}
\caption{Example \ref{sec:ex1}: Test errors over 1000 random parameters for RB--DeepONet, 
POD--DeepONet, FEONet, and RB--Galerkin.}
\label{tab:hc:new-agg-compact}
\end{table}

\begin{figure}[htbp]
\centering
\includegraphics[width=.98\textwidth]{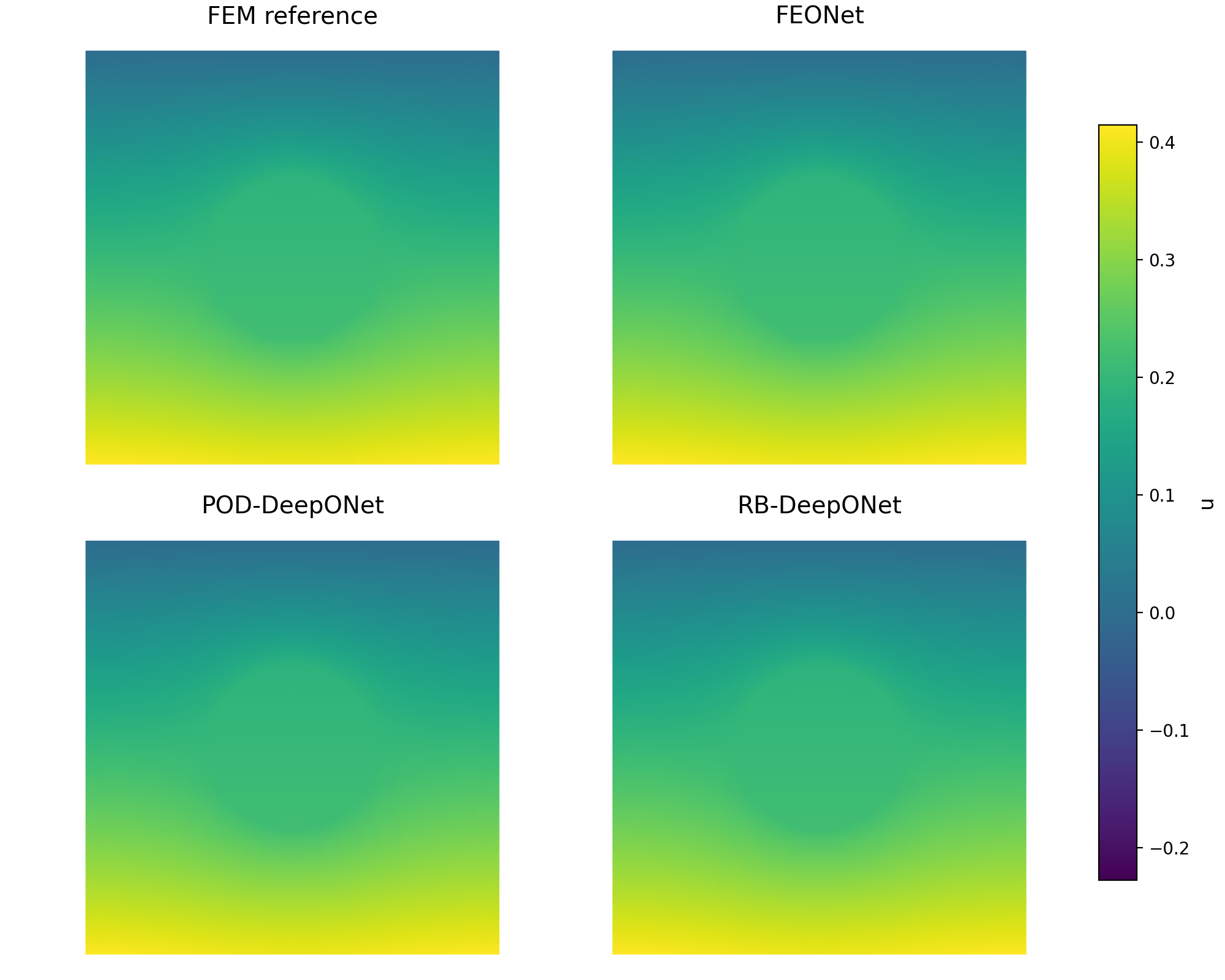}
\caption{Example \ref{sec:ex1}: Representative-parameter comparison at $\mathbf{k}=(k_1,k_2)=(10,\,0.5)$ of (top–left) FEM reference, (top–right) FEONet,
(bottom–left) POD--DeepONet, and
(bottom–right) RB--DeepONet.}
\label{fig:hc:viz}
\end{figure}

\begin{figure}[htbp]
    \centering
    \subfigure[FEONet]{%
        \label{fig:loss-fem}
        \includegraphics[width=0.31\textwidth]{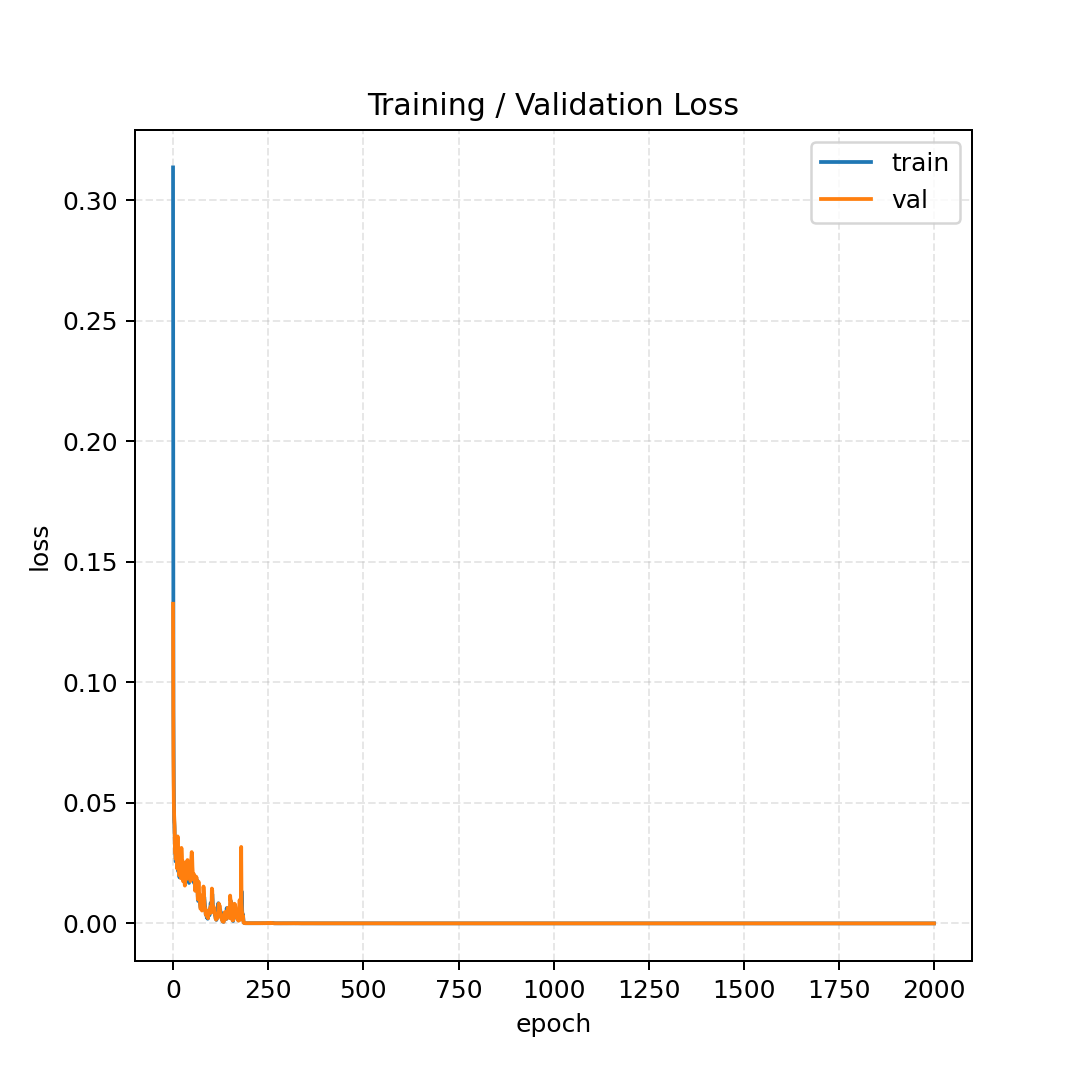}
    }%
    \subfigure[POD--DeepONet]{%
        \label{fig:loss-podN3}
        \includegraphics[width=0.31\textwidth]{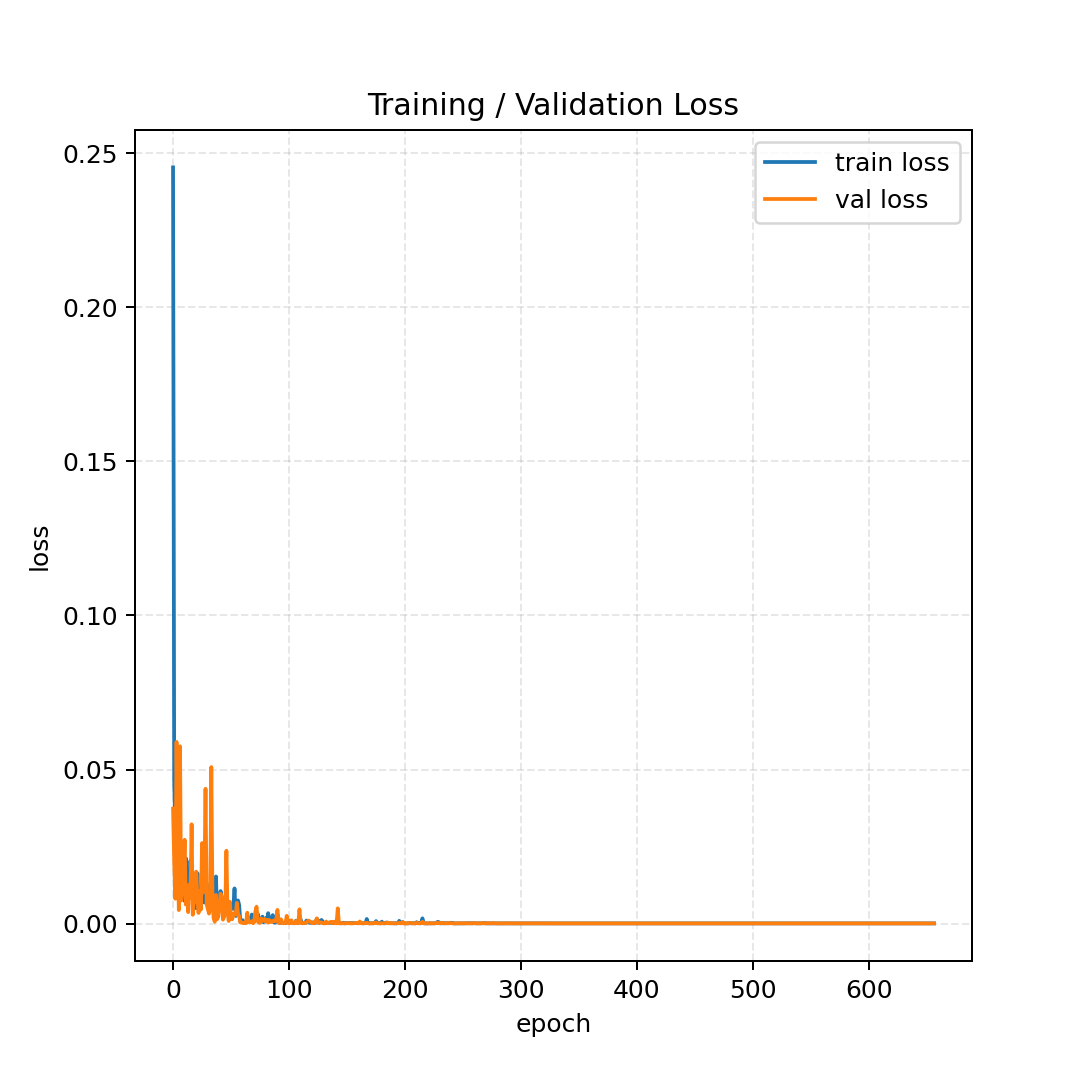}
    }%
    \subfigure[RB--DeepONet]{%
        \label{fig:loss-greedyN3}
        \includegraphics[width=0.31\textwidth]{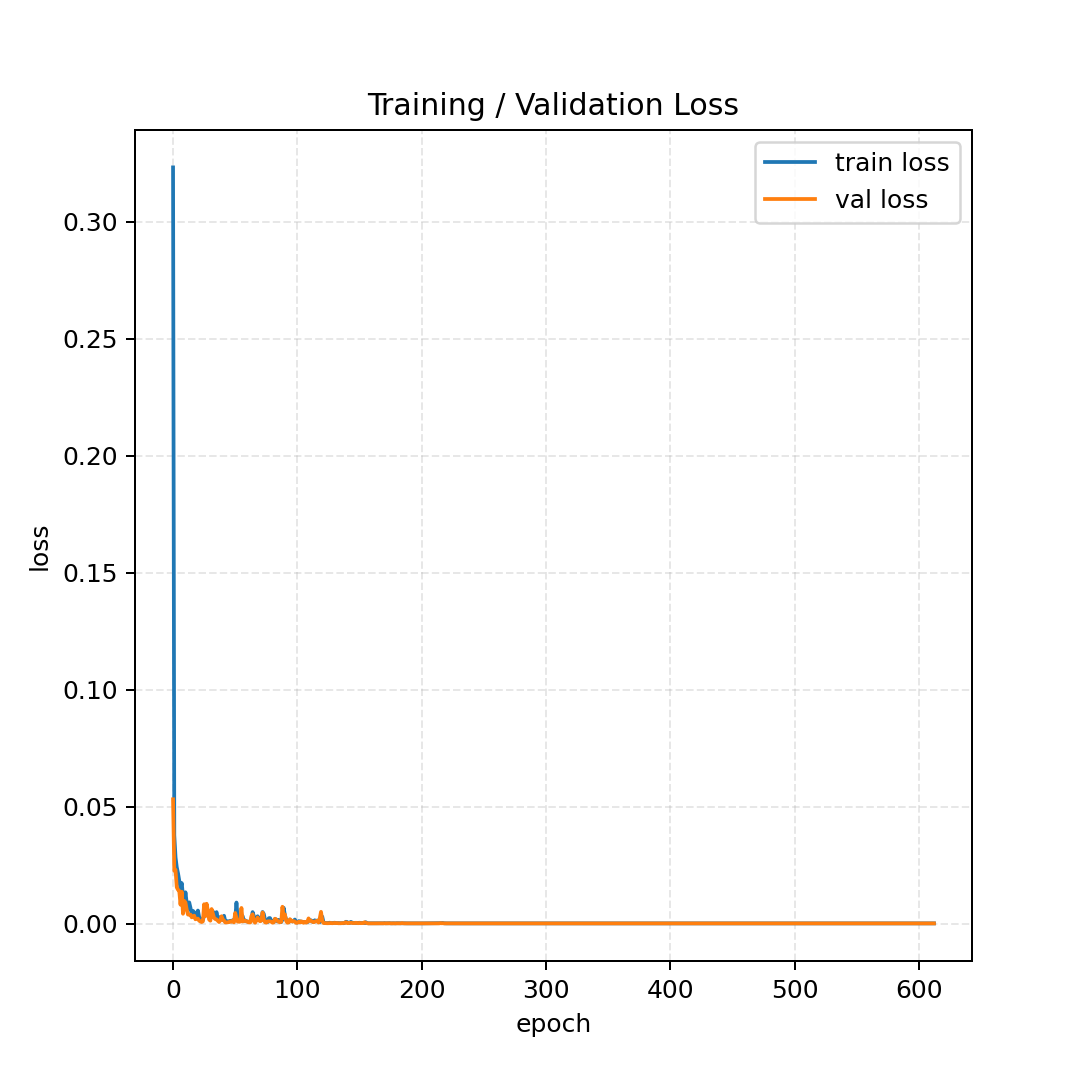}
    }%
    \caption{Example \ref{sec:ex1}: Training/validation loss.}
    \label{fig:hc:loss}
\end{figure}

\subsection{Example 2: Parametric diffusion–reaction with independently varying data.}
\label{sec:ex2}

Second, we consider the family of second–order elliptic problems.
The strong form is
\begin{equation}\label{eq:mixed:model}
  \begin{cases}
  \mathcal L(\mathbf k)\,u(\mathbf x;\mathbf k)=f(\mathbf x), & \mathbf x\in\Omega,\\[2pt]
  u(\mathbf x;\mathbf k)=g_D(\mathbf x), & \mathbf x\in\Gamma_D,\\[2pt]
  \kappa_0\,\nabla u(\mathbf x;\mathbf k)\!\cdot \mathbf{n} = h_N(\mathbf x), & \mathbf x\in\Gamma_N,\\[2pt]
  \kappa_0\,\nabla u(\mathbf x;\mathbf k)\!\cdot \mathbf{n} + \beta_0\,u(\mathbf x;\mathbf k)
  = r_R(\mathbf x), & \mathbf x\in\Gamma_R,
  \end{cases}
\end{equation}
Here, $\Omega:=(0,1)^2$ with the boundary partition
\[
\Gamma_D:=\{x=0\}\cup\{y=1\},\qquad
\Gamma_N:=\{y=0\},\qquad
\Gamma_R:=\{x=1\},
\]
and the operator is defined as
\begin{equation}\label{eq:mixed:op}
  \mathcal L(\mathbf k):=\kappa_0(-\Delta)+\alpha_0 I,
  \qquad \kappa_0>0,\ \alpha_0\ge 0.
\end{equation}
Here, $\mathbf{k}:=[\kappa_0,\alpha_0,\beta_0]^\top$ is the parameter.

To obtain consistent boundary and source pairs, we prescribe a smooth solution family
\begin{equation}\label{eq:mixed:ustar}
  \begin{aligned}
  u_\star(x,y;\boldsymbol{\xi})
  &= a_1\sin(\pi x)\sin(\pi y)
   + a_2\sin(2\pi x)\sin(\pi y)
   + a_3\exp\!\Big(-\tfrac{(x-x_c)^2+(y-y_c)^2}{2\sigma^2}\Big)\\
  &\quad + a_4\cos(\pi x)\sinh(y-\tfrac12),
  \end{aligned}
\end{equation}
and introduce an auxiliary vector $\boldsymbol{\xi}:=[a_1,a_2,a_3,a_4,x_c,y_c,\sigma]^\top$ to modulate the boundary traces and the distributed source independently of $\mathbf{k}$.
With $\mathbf{k}$ governing the operator coefficients, this choice isolates exogenous data variability and directly tests the boundary- and source-mode mechanisms of RB--DeepONet.
We sample the parameters independently as
\[
  a_1,a_2,a_4\sim\mathcal U[-1,1],\qquad
  a_3\sim\mathcal U[0,1],\qquad
  (x_c,y_c)\sim\mathcal U([0.2,0.8]^2),\qquad
  \sigma\sim\mathcal U[0.05,0.2],
\]
\[
  \kappa_0\sim\mathcal U[0.5,2],\qquad
  \alpha_0\sim\mathcal U[0,2],\qquad
  \beta_0\sim\mathcal U[0,10].
\]
The data $(f,g_D,h_N,r_R)$ are then defined by
\begin{equation}\label{eq:mixed:MMS}
\begin{aligned}
    f(\cdot):=& \mathcal L(\mathbf k)\,u_\star(\cdot;\boldsymbol{\xi}),\quad
  g_D(\cdot):=u_\star(\cdot;\boldsymbol{\xi})|_{\Gamma_D},\\
  h_N(\cdot):=&\kappa_0\,\nabla u_\star(\cdot;\boldsymbol{\xi})\!\cdot \mathbf{n}\big|_{\Gamma_N},\quad
  r_R(\cdot):=\kappa_0\,\nabla u_\star(\cdot;\boldsymbol{\xi})\!\cdot \mathbf{n}+\beta_0\,u_\star(\cdot;\boldsymbol{\xi})\big|_{\Gamma_R}.
\end{aligned} 
\end{equation}

We follow Section~\ref{sec:Specialization for Case2} to construct boundary and source modes, and precompute the RB trunk and reduced
operators.
The domain $\Omega=(0,1)^2$ is discretized by a $64\times64$ structured
triangulation, giving $N_h=64^2=4096$ nodal unknowns.
The Dirichlet boundary is $\Gamma_D=\{x=0\}\cup\{y=1\}$, containing
$|\mathcal B_D|=127$ nodes, so the number of free degrees of freedom is
$N_0=4096-127=3969$.

For the offline stage, we draw $N_k=10{,}000$ parameter–data samples and
split them $80\%/20\%$ into training and validation sets for learning the
branch networks.
We fix the reference parameter for the energy norm and lifting as $\mathbf{k}_\star=[1.2,0.6,1.0]^\top$.
A uniform tolerance $10^{-7}$ is used for the boundary, source, and trunk
constructions, which yields $(N,r_f,r_g)=(209,128,16)$.
For RB--DeepONet and POD--DeepONet, we target these same ranks to ensure a
fair comparison.

Strictly speaking, both POD--DeepONet and FEONet are designed to take as input the full-order discretizations of the data $(f,g_D,h_N,r_R)$ together with the physical parameters, which in this example would lead to prohibitively high-dimensional inputs. Therefore, we instead work with the low-dimensional augmented feature vector $\mathbf{k}_{\rm aug}$ and the approximate right-hand side obtained by projecting the data onto the boundary and source modes, so that all three networks use $\mathbf{k}_{\rm aug}$ as input while the variational residuals are evaluated with $\mathbf{F}_{\rm rb}(\mathbf{k}_{\rm aug})$ in \eqref{eq:Frb-clean} and the corresponding full-order matrices.
Thus, the branch input is the feature vector in \eqref{eq:feature vector}
of dimension $3+r_f+r_g=147$, and the branch output has dimension $N=209,3969$ for RB-- and POD--DeepONets and FEONet, respectively.

We train FEONet for $50{,}000$ epochs, whereas RB--DeepONet and
POD--DeepONet are trained for $2000$ epochs.
Notably, RB-- and POD--DeepONet output only $N=209$ coefficients, versus
$N_0=3969$ for FEONet, about $19$ times smaller, while also requiring
$25$ times fewer training epochs, highlighting their lightweight design
and faster, more stable training.
All remaining settings (branch architecture, optimizer, learning-rate
schedule, residual loss, and batch protocol) follow
Example~\ref{sec:ex1}.

Figure~\ref{fig:hc:loss2} shows the training/validation losses.
All three networks eventually converge, with POD--DeepONet and RB--DeepONet
reaching a low residual level after roughly $10^3$ epochs, whereas FEONet
requires substantially more epochs and exhibits larger fluctuations.
Figure~\ref{fig:hc:viz 2} displays a representative parameter:
the ground truth solution, FEONet, POD--DeepONet, and RB--DeepONet
are visually indistinguishable, indicating that all surrogates reproduce
the mixed boundary conditions and spatially varying load at the plotting
scale.
Table~\ref{tab:rbdo-feonet-rbg} reports mean and 95th-percentile errors
over $1000$ test parameters.
The RB--Galerkin reference again attains small errors (mean rel-$L^2$
$\approx 1.4\times10^{-3}$), confirming that the RB trunk and the
boundary/source modes are accurate for this Case~II setting.
Among the learned models, POD--DeepONet and RB--DeepONet both achieve
mean rel-$L^2$ of order $10^{-2}$ with 95th percentiles below
$6\times10^{-2}$.
POD--DeepONet is slightly more accurate in rel-$L^2$ and rel-energy,
whereas RB--DeepONet remains within a modest factor while relying only on
residual information and operating entirely in the reduced space.
FEONet attains the smallest rel-$L^2$, but at the cost of predicting
$N_0=3969$ coefficients instead of $N=209$ reduced coordinates.

Overall, this example demonstrates that RB--DeepONet can handle exogenous
boundary and source data that vary independently of the physical
parameters and, with a compact RB trunk, delivers relative $L^2$ errors
of about $1\%$ while avoiding online PDE solves and high-dimensional
FE outputs.

\begin{figure}[htbp]
    \centering
    \subfigure[FEONet]{%
        \label{fig:loss-fem 2}
        \includegraphics[width=0.31\textwidth]{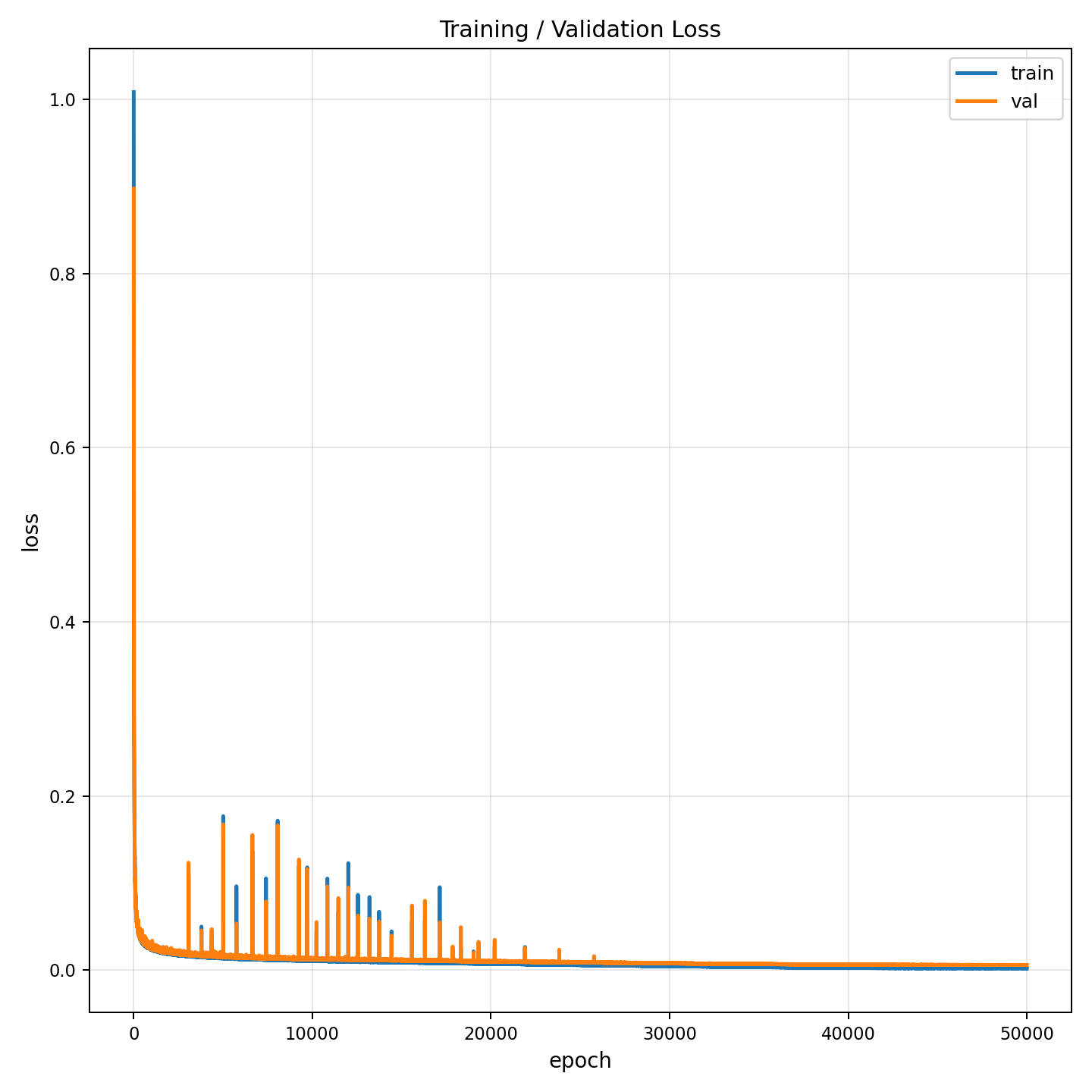}
    }%
    \subfigure[POD--DeepONet]{%
        \label{fig:loss-podN3 2}
        \includegraphics[width=0.31\textwidth]{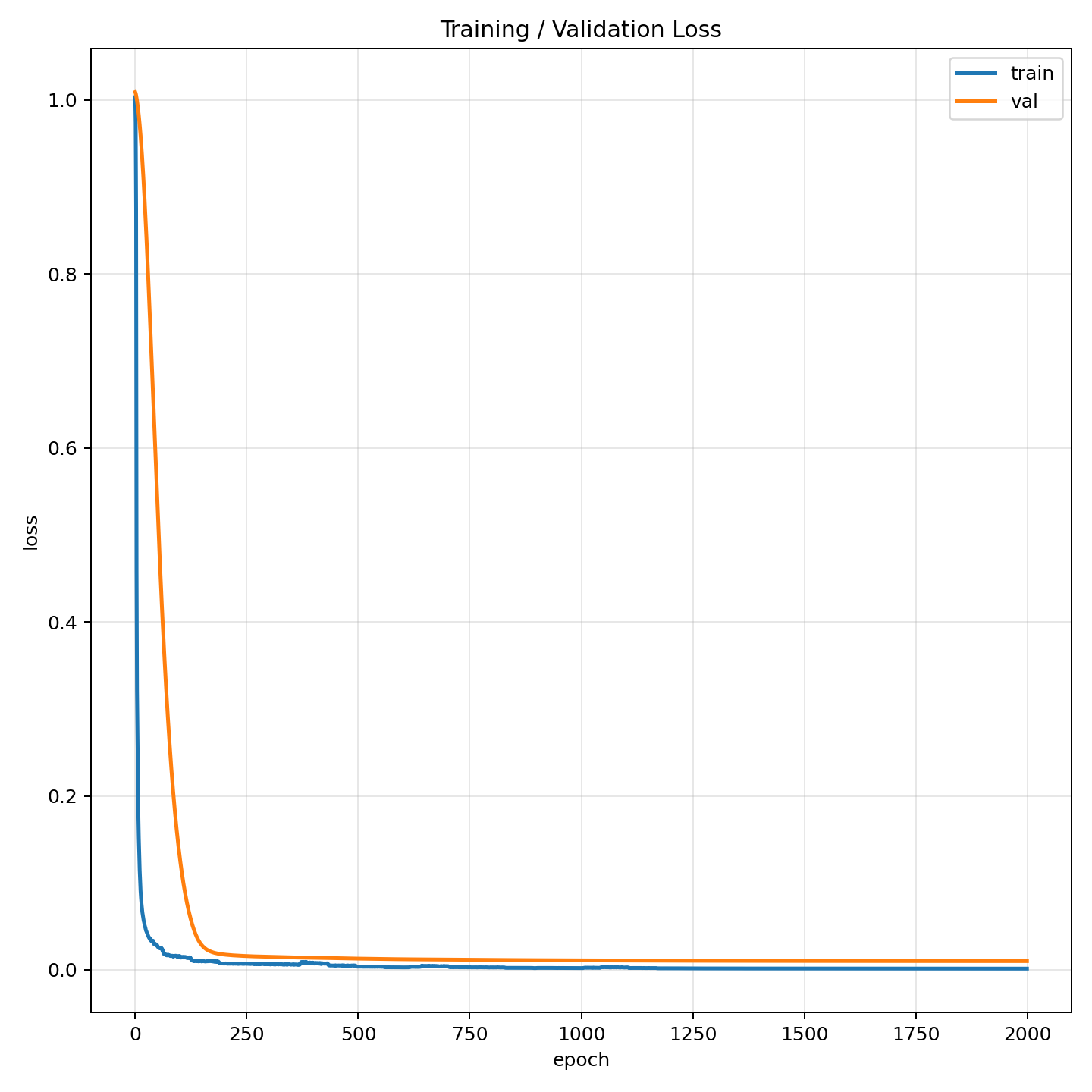}
    }%
    \subfigure[RB--DeepONet]{%
        \label{fig:loss-greedyN3 2}
        \includegraphics[width=0.31\textwidth]{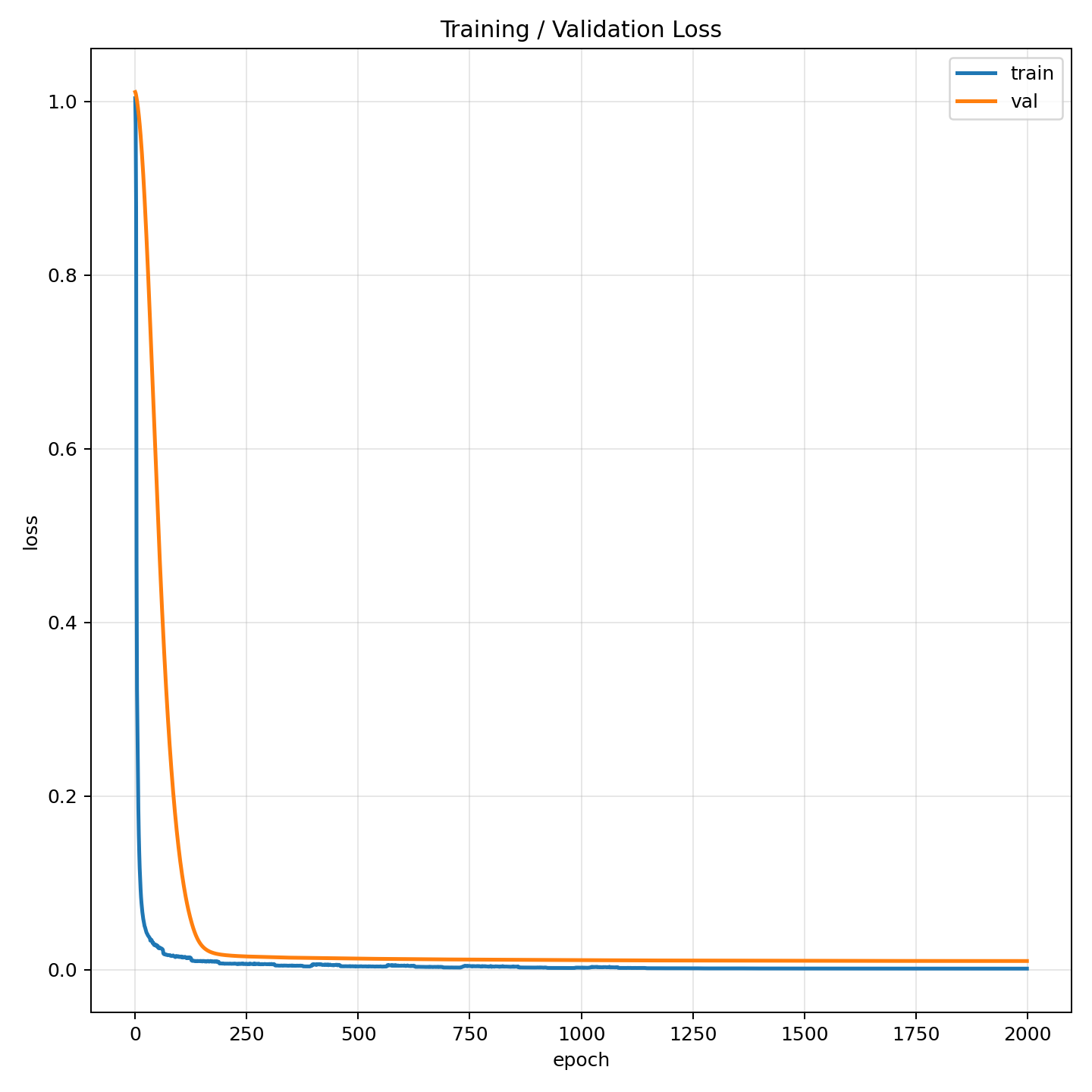}
    }%
    \caption{Example \ref{sec:ex2}: Training/validation loss for FEONet and POD--, RB--DeepONet with $(N,r_f,r_g)=(209,128,16)$.}
    \label{fig:hc:loss2}
\end{figure}

\begin{figure}[htbp]
\centering
\includegraphics[width=.88\textwidth]{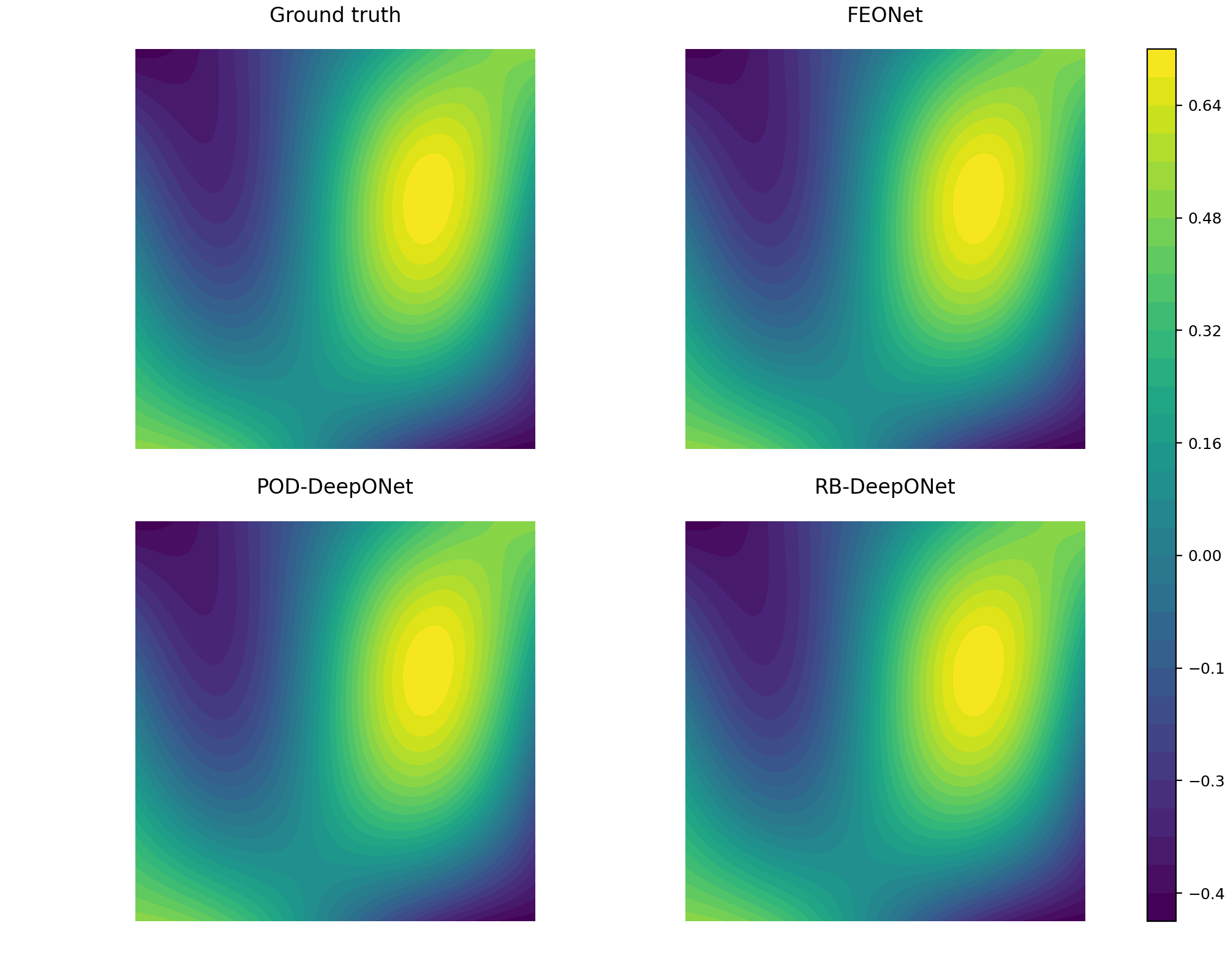}
\caption{Example \ref{sec:ex2}: Representative-parameter comparison of 
(top–left) ground truth, (top–right) FEONet, 
(bottom–left) POD--DeepONet, and (bottom–right) RB--DeepONet.}
\label{fig:hc:viz 2}
\end{figure}

\begin{table}[htbp]
\centering
\footnotesize
\setlength{\tabcolsep}{6pt}
\begin{tabular}{lccc}
\toprule
\textbf{Trunk} &
\textbf{rel-$L^2$ (mean / $p95$)} &
\textbf{rel-energy (mean / $p95$)} &
\textbf{rel-residual (mean / $p95$)} \\
\midrule
RB--DeepONet    & $\mathbf{1.17\times10^{-2}\ /\ 4.33\times10^{-2}}$ & $\mathbf{1.53\times10^{-2}\ /\ 5.88\times10^{-2}}$ & $\mathbf{1.38\times10^{-2}\ /\ 5.08\times10^{-2}}$ \\[2pt]
POD--DeepONet & $9.47\times10^{-3}\ /\ 2.37\times10^{-2}$ & $1.06\times10^{-2}\ /\ 4.18\times10^{-2}$ & $9.68\times10^{-3}\ /\ 4.33\times10^{-2}$ \\[2pt]
FEONet        & $2.18\times10^{-3}\ /\ 4.68\times10^{-3}$ &
                               $4.88\times10^{-3}\ /\ 1.39\times10^{-2}$ &
                               $7.15\times10^{-16}\ /\ 7.94\times10^{-16}$ \\[2pt]
RB--Galerkin                & $1.41\times10^{-3}\ /\ 4.50\times10^{-3}$ &
                               $5.99\times10^{-3}\ /\ 2.33\times10^{-2}$ &
                               $4.59\times10^{-16}\ /\ 8.15\times10^{-16}$ \\
\bottomrule
\end{tabular}
\caption{Example \ref{sec:ex2}: Test errors (mean / $p95$) over 1000 random parameters.}

\label{tab:rbdo-feonet-rbg}
\end{table}

\subsection{Example 3: Non-affine heat conduction}
\label{sec:ex3}

In this example, we demonstrate that our framework can extend seamlessly to non-affine operators via empirical interpolation (EIM). Once the RB trunk is sufficiently expressive, the RB--DeepONet achieves accuracy comparable to the RB--Galerkin solution and the full FEM reference. We report results with the same training budget and evaluation protocol as in Example~\ref{sec:ex1}.

The geometry and boundary conditions follow Example~\ref{sec:ex1}. In addition to the parameters $k_1,k_2$, we parameterize the inclusion radius by $k_3$ and denote the disk by $\Omega_0(k_3)$. The thermal conductivity is
\[
\kappa(\mathbf{x};k_1,k_3)=
\begin{cases}
k_1, & \mathbf{x}\in\Omega_0(k_3),\\
1,   & \mathbf{x}\in\Omega_1:=\Omega\setminus\Omega_0(k_3).
\end{cases}
\]
We write $\mathbf{k}=[k_1,k_2,k_3]^\top$ and aim to learn the mapping $\mathbf{k}\mapsto u(\mathbf{x};\mathbf{k})$, where $u(\cdot;\mathbf{k})$ solves \eqref{eq:strong-ex1} with $\kappa(\cdot;k_1)$ replaced by $\kappa(\cdot;k_1,k_3)$. The bilinear form reads
\[
a(u,v;\mathbf{k})=\int_{\Omega}\kappa(\mathbf{x};k_1,k_3)\,\nabla u\cdot\nabla v\,\mathrm d\mathbf{x},
\]
so the coefficient depends non-affinely on $\mathbf{k}$. We recover an approximate affine separation by EIM as in \cite{hesthaven2016certified}.

Let $\widehat\Omega:=\widehat\Omega_1\cup\widehat\Omega_0$ be a fixed reference domain with reference inclusion radius $r_0=0.2$ (as in Example~\ref{sec:ex1}). We introduce the continuous radial map
\[
\mathbf T_r(\widehat{\mathbf{x}}):=\varphi_r(|\widehat{\mathbf{x}}|)\,\widehat{\mathbf{x}},\qquad
\widehat{\mathbf{x}}\in\widehat\Omega,
\]
with
\[
\varphi_r(\rho)=
\begin{cases}
1, & 0\le \rho< r_{-},\\[4pt]
\dfrac{r_0(r_0-\rho)+r(\rho-r_{-})}{r_0(r_0-r_{-})}, & r_{-}\le \rho<r_0,\\[10pt]
\dfrac{r_0(r_0-\rho)+r(\rho-r_{+})}{r_0(r_0-r_{+})}, & r_0\le \rho<r_{+},\\[10pt]
1, & r_{+}\le \rho,
\end{cases}
\]
for some $0<r_{-}\le r_{\min}<r_{\max}<r_{+}<1$ and $r\in[r_{\min},r_{\max}]$. Here, $\rho:=|\widehat{\mathbf{x}}|$. Pushing the PDE to $\widehat\Omega$ yields: find $\widehat u(\mathbf{k})\in V$ such that
\[
\widehat a(\widehat u,\widehat v;\mathbf{k})=\widehat \ell(\widehat v;\mathbf{k})\quad\forall\,\widehat v\in V,
\]
where
\[
\begin{aligned}
\widehat a(\widehat u,\widehat v;\mathbf{k})
&:=\int_{\widehat\Omega_1}\nabla\widehat u\cdot\big(\mathbf G(\widehat{\mathbf{x}};k_3)\nabla\widehat v\big)\,\mathrm d\widehat{\mathbf{x}}
\;+\;
k_1\int_{\widehat\Omega_0}\nabla\widehat u\cdot\big(\mathbf G(\widehat{\mathbf{x}};k_3)\nabla\widehat v\big)\,\mathrm d\widehat{\mathbf{x}},\\
\widehat \ell(\widehat v;\mathbf{k})
&:=k_2\int_{\widehat\Gamma_{\mathrm{base}}}\widehat v\,\mathrm d\widehat s,
\end{aligned}
\]
and
\[
\mathbf G(\widehat{\mathbf{x}};k_3):=\bigl|\det\widehat {\mathbf{J}}(\widehat{\mathbf{x}};k_3)\bigr|\,
\widehat {\mathbf{J}}(\widehat{\mathbf{x}};k_3)^{-\top}\widehat {\mathbf{J}}(\widehat{\mathbf{x}};k_3)^{-1},\qquad
\bigl(\widehat {\mathbf{J}}(\widehat{\mathbf{x}};k_3)\bigr)_{ij}:=\frac{\partial (\mathbf T_{k_3}(\widehat{\mathbf{x}}))_j}{\partial \widehat{\mathbf{x}}_i},\quad i,j=1,2.
\]
We apply the empirical interpolation algorithm for vector fields \cite[Chap.~5]{hesthaven2016certified} to $\mathbf G(\cdot;k_3)$ and obtain
\begin{equation}\label{eq:ex3-eim}
\mathbf G(\widehat{\mathbf{x}};k_3)\approx \sum_{q=1}^{Q}\alpha_q(k_3)\,\mathbf H_q(\widehat{\mathbf{x}}),
\end{equation}
with parameter-independent tensors $\{\mathbf H_q\}_{q=1}^Q$ and online coefficients $\alpha_q(k_3)$ determined from EIM pivot values. The approximation exhibits an exponential decay as $Q$ increases (see \cite{hesthaven2016certified}).

Let $\Psi=[\psi_1,\dots,\psi_N]$ be the RB space obtained by Greedy selection, represented in the FE basis. Using \eqref{eq:ex3-eim}, the reduced operators take the form
\begin{equation}\label{eq:ex3-rb-ops}
\mathbf A_{\mathrm{rb}}(\mathbf{k})=\sum_{q=1}^{Q}\alpha_q(k_3)\,\mathbf H_{\mathrm{rb},q},
\qquad
\mathbf F_{\mathrm{rb}}(\mathbf{k})=\Psi^\top \mathbf F,
\qquad
\mathbf H_{\mathrm{rb},q}:=\Psi^\top \mathbf H_q\,\Psi.
\end{equation}
In the offline stage, we assemble the $Q$ reduced blocks $\{\mathbf H_{\mathrm{rb},q}\}_{q=1}^{Q}$. Then, in the online stage, we no longer touch any full-order matrices. RB--DeepONet fixes $\Psi$ as the trunk and trains only the branch $\mathbf c_\theta(\mathbf{k})\in\mathbb{R}^N$ with the residual loss \eqref{eq:forward-loss} that vanishes exactly at the RB--Galerkin coefficients. The prediction is $u_\theta(\mathbf{k})=\Psi\,\mathbf c_\theta(\mathbf{k})$.

We consider
\[
k_1\in[0.1,10],\qquad k_2\in[-1,1],\qquad k_3\in[0.05,0.45].
\]
From $4000$ random samples we build POD trunks with tolerance $\epsilon_{\mathrm{POD}}=10^{-7}$, resulting in $N=5$. For comparison, we also construct the Greedy RB trunk with $N=5$. For the geometry surrogate, we take $Q=15$, which yields an average relative $L^2$ error of about $10^{-4}$ against direct FEM assembled without the affine approximation as shown in Table \ref{tab:ex3}.
All the other network settings are the same as Example \ref{sec:ex1}.

Figure~\ref{fig:loss ex3} shows the training and validation losses.
Both POD--DeepONet and RB--DeepONet converge rapidly to a small residual
level with $N=5$ trunk modes.
Figure~\ref{fig:compare ex3} fixes $\mathbf k=(6.68,0.94,0.2)$ and compares
the FEM solution, the RB--Galerkin reference, POD--DeepONet, and
RB--DeepONet. All RB-based surrogates are visually indistinguishable from
the FEM field.
Table~\ref{tab:ex3} reports the mean and 95th-percentile errors over
$1000$ test parameters.
The RB--Galerkin solution is essentially exact (mean rel-$L^2$ 
$\approx 1.6\times10^{-4}$), confirming that the RB space obtained from
the EIM-assembled operators \eqref{eq:ex3-rb-ops} is highly accurate even
for this non-affine problem.
Among the learned models, POD--DeepONet attains mean rel-$L^2$ around
$9\times10^{-3}$ and RB--DeepONet around $2.3\times10^{-2}$, with
95th-percentile errors below $3\times10^{-2}$ and $3.3\times10^{-2}$,
respectively.
Thus, RB--DeepONet remains within a few percent of the FEM solution while
relying only on residual information and never using solution labels,
whereas POD--DeepONet requires full-order snapshots for supervised
training.
Both networks share the same low online complexity, but the Greedy
RB trunk can be built from far fewer truth solves than the POD trunk.

Overall, this example shows that the EIM-based reduced operators yield a
separable assembly with negligible online cost, and that RB--DeepONet can
learn the RB--Galerkin map robustly in this non-affine setting, achieving
POD-level accuracy with a fully reduced, label-free training pipeline.

\begin{figure}[htbp]
    \centering
    \subfigure[POD--DeepONet ($N{=}5$)]{%
        \label{fig:loss-pod ex3}
        \includegraphics[width=0.4\textwidth]{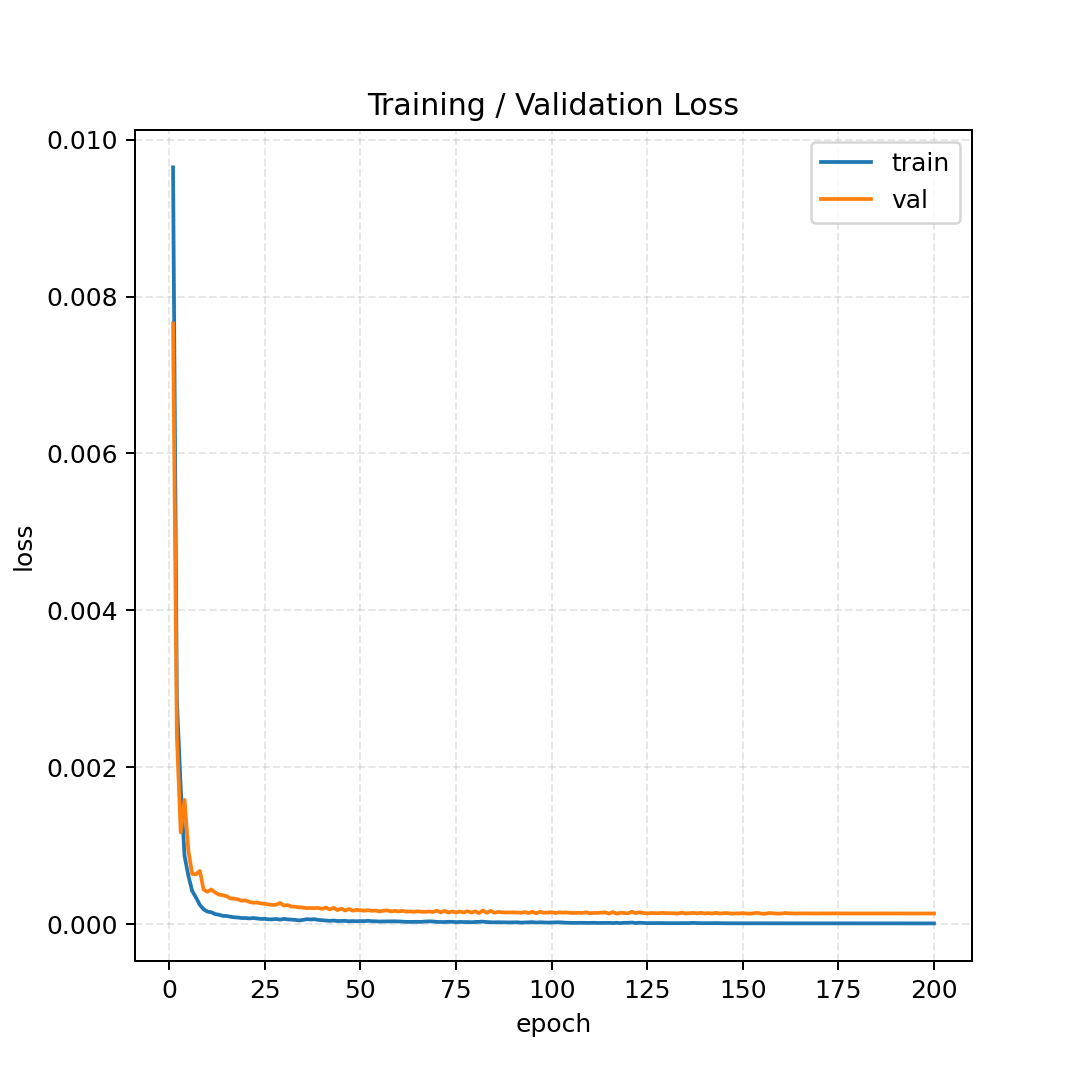}
    }%
    \subfigure[RB--DeepONet ($N{=}5$)]{%
        \label{fig:loss-greedy ex3}
        \includegraphics[width=0.4\textwidth]{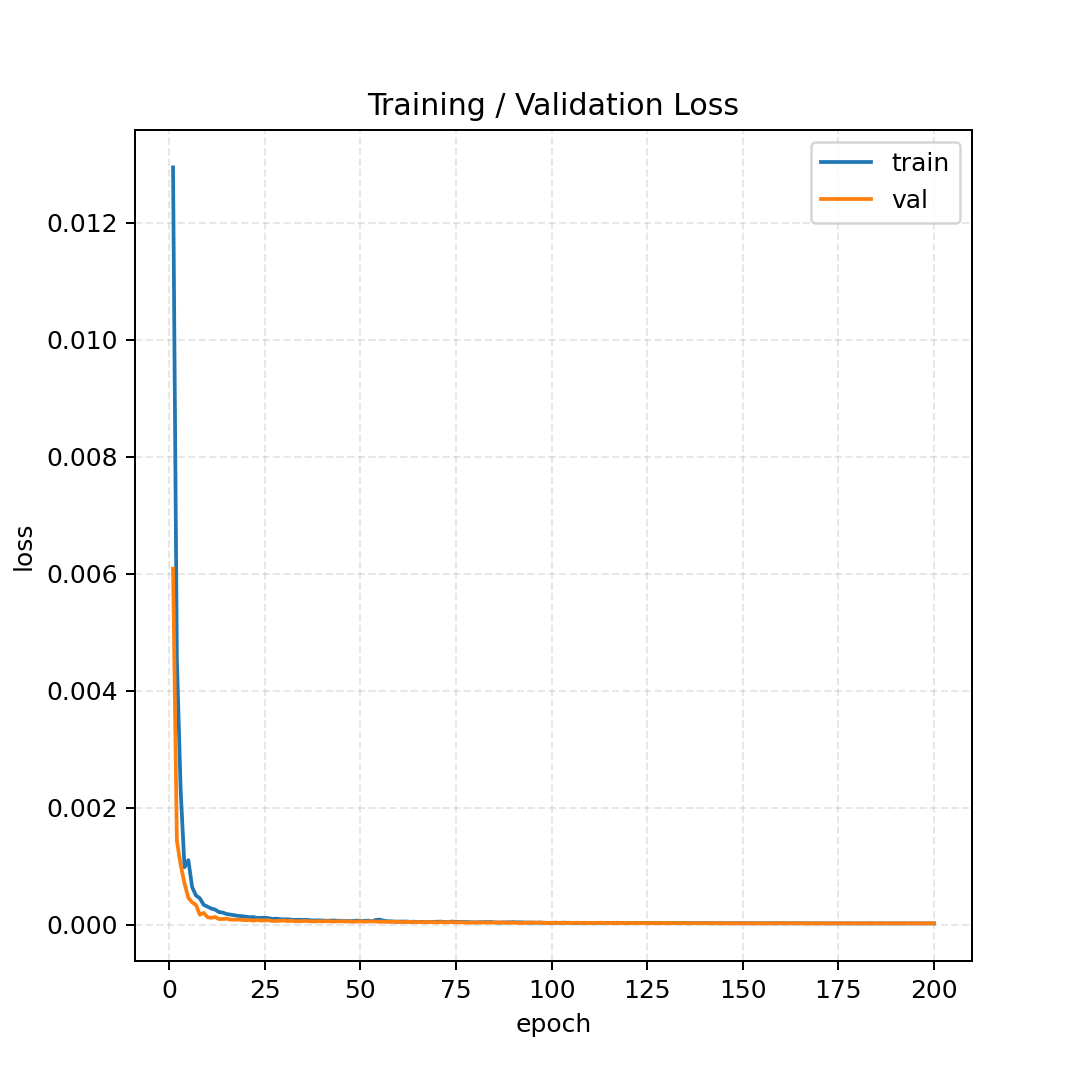}
    }%
    \caption{Example \ref{sec:ex3}: Training/validation loss.}
\label{fig:loss ex3}
\end{figure}

\begin{figure}[htbp]
\centering
\includegraphics[width=.88\textwidth]{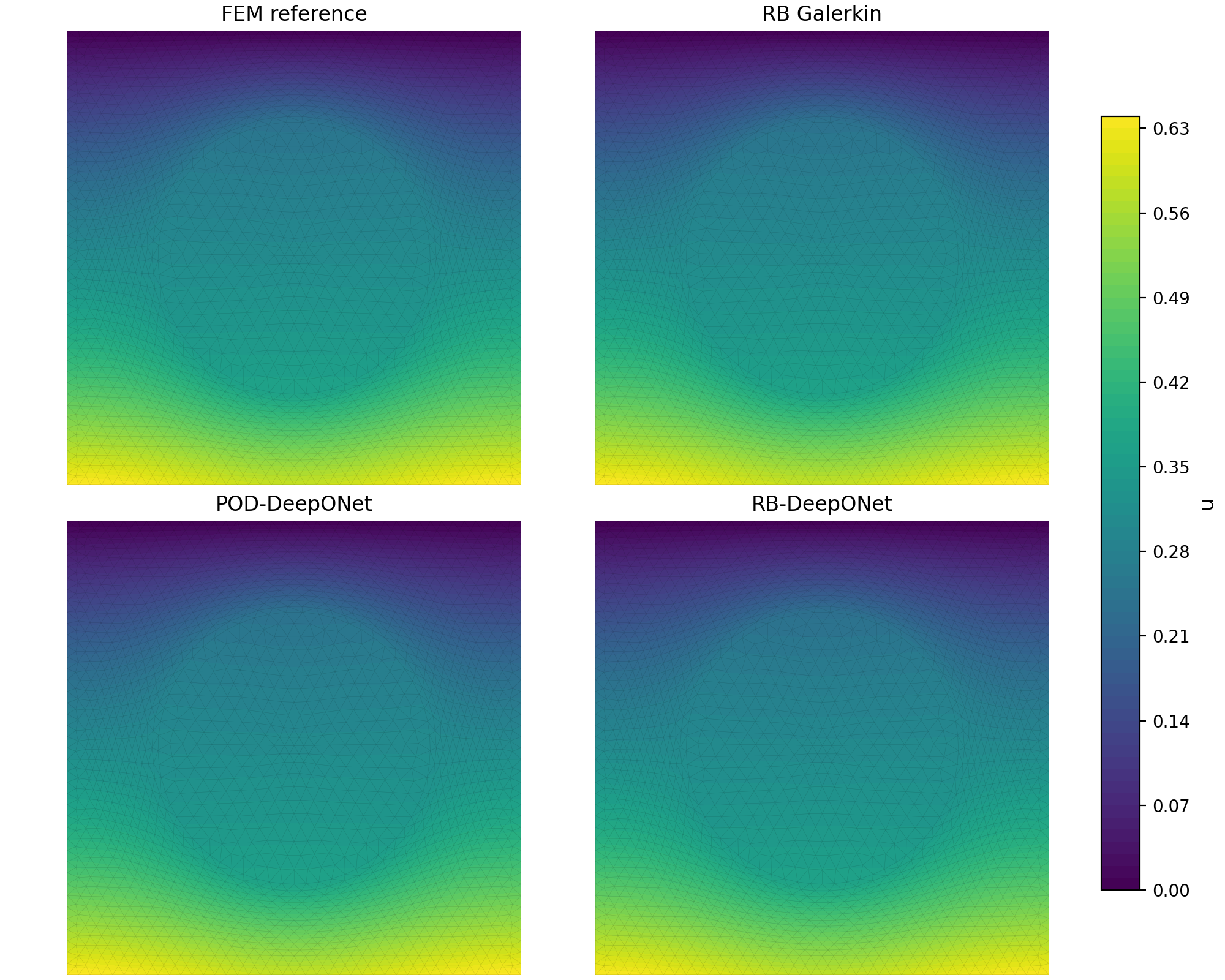}
\caption{Example \ref{sec:ex3}: Representative-parameter comparison of 
(top–left) FEM reference, (top–right) RB--Galerkin, 
(bottom–left) POD--DeepONet, and (bottom–right) RB--DeepONet.}
\label{fig:compare ex3}
\end{figure}

\begin{table}[htbp]
\centering
\footnotesize
\setlength{\tabcolsep}{6pt}
\begin{tabular}{lccc}
\toprule
\textbf{Trunk} &
\textbf{rel-$L^2$ (mean / $p95$)} &
\textbf{rel-energy (mean / $p95$)} &
\textbf{rel-residual (mean / $p95$)} \\
\midrule
RB--DeepONet & $\mathbf{2.28\times10^{-2}\ /\ 3.27\times10^{-2}}$ & $\mathbf{3.74\times10^{-2}\ /\ 4.79\times10^{-2}}$ & $\mathbf{5.48\times10^{-2}\ /\ 6.84\times10^{-2}}$ \\[2pt]
POD--DeepONet    & $9.14\times10^{-3}\ /\ 2.08\times10^{-2}$ & $1.82\times10^{-2}\ /\ 3.81\times10^{-2}$ & $1.71\times10^{-2}\ /\ 4.71\times10^{-2}$ \\[2pt]
RB--Galerkin                & $1.57\times10^{-4}\ /\ 4.83\times10^{-4}$ &
                               $6.14\times10^{-4}\ /\ 8.73\times10^{-4}$ &
                               $1.05\times10^{-15}\ /\ 3.23\times10^{-15}$ \\
\hline
\end{tabular}
\caption{Example \ref{sec:ex3}: Test errors (mean / $p95$) over 1000 random parameters.}

\label{tab:ex3}
\end{table}

Note that unlike Example~\ref{sec:ex1}, here we do not include FEONet as a baseline in this non–affine setting.
The reason is due to the prohibitive training cost arising from per–sample full–order operator application: the original FEONet formulation \cite{lee2025finite} treats variable coefficients and boundary data as network inputs and applies the full operator during training on FE basis, without an offline affine surrogate. 
Consequently, when the operator changes with $\mathbf{k}=[k_1,k_2,k_3]^\top$, each training step requires assembling the full FEM operator $\mathbf A(\mathbf k)$ at dimension $N_h$, which makes the per–epoch cost scale with $N_h$ and the batch size. 
In contrast, our EIM–augmented RB formulation yields an online assembly whose evaluation cost scales only with the reduced dimension $N$ and the EIM rank $Q$. 
To isolate the effect of the trunk space and to enable a fair and affordable comparison, we therefore benchmark only against the full FEM reference and the RB--Galerkin solution on the same reduced space. 

Table~\ref{tab:trainable_params_offline} summarizes the network sizes and offline costs of the three surrogates.
Across all examples, RB--DeepONet uses roughly $2\times 10^5$ trainable parameters, comparable to POD--DeepONet but more than four times smaller than FEONet in Example \ref{sec:ex2}.
At the same time, the number of offline truth solves required by RB--DeepONet is reduced by two to three orders of magnitude compared with POD--DeepONet.
FEONet avoids offline solves but pays for this with a much larger network and the need to evaluate full-order residuals during training.
Therefore, RB--DeepONet achieves a favorable balance between sample efficiency, model size, and training cost.

\begin{table}[htbp]
  \centering\footnotesize
  \caption{Trainable parameters and offline truth solves
  for three examples. Here $d_{\text{in}}$ and $d_{\text{out}}$
  denote the input and output dimensions of the branch network, respectively.}
  \label{tab:trainable_params_offline}
  \begin{threeparttable}
  \begin{tabular}{llrrr}
    \toprule
    Example & Method & $d_{\text{in}}$ & \# trainable params & Offline truth solves \\
    \midrule
    \ref{sec:ex1}& RB--DeepONet  & $2$   & $198{,}915$   & $3$ \\
        & POD--DeepONet & $2$   & $198{,}915$   & $2{,}100$ \\
        & FEONet        & $2$   & $729{,}106$   & $0$ \\
    \midrule
    \ref{sec:ex2} & RB--DeepONet  & $147$ & $288{,}977$   & $337$\tnote{a} \\
        & POD--DeepONet & $147$ & $288{,}977$   & $10{,}000$ \\
        & FEONet        & $147$ & $1{,}255{,}297$ & $0$ \\
    \midrule
    \ref{sec:ex3} & RB--DeepONet  & $3$   & $199{,}685$   & $5$ \\
        & POD--DeepONet & $3$   & $199{,}685$   & $4{,}000$ \\
    \bottomrule
  \end{tabular}
\begin{tablenotes}
  \footnotesize
  \item[a] The construction of the source modes requires additionally solving the Riesz–map
  problems for each of the $r_f=128$ source modes, so the total number of
  full-order solves used by RB--DeepONet in this example is
  $N + r_f = 337$.
\end{tablenotes}
  \end{threeparttable}
\end{table}

\section{Conclusions and future work}\label{sec:Conclusions and future work}

In this work, we introduced RB--DeepONet, a residual–driven operator–learning framework. The DeepONet trunk is fixed to a rigorously constructed reduced–basis (RB) space, and the branch network predicts the RB coefficients. This gives a principled interface between certified model reduction and operator learning: RB--DeepONet preserves the structure and stability of RB--Galerkin formulations while enabling fast, label–free coefficient prediction through a residual loss.
We developed two specializations. In Case~I, with fully parameterized coefficients and data, the RB right–hand side is assembled via a standard affine decomposition. In Case~II, with operator parametrization and independently varying data, boundary and source functions are compressed into boundary and source modes and incorporated through lifting, without changing the training objective.
On the analytical side, we established well–posedness of the RB system and generalization bounds for the residual loss. We also derived error estimates that quantify the impact of modal truncation. On the computational side, three benchmark problems showed that RB--DeepONet attains accuracy comparable to intrusive RB--Galerkin and to POD--DeepONet and FEONet surrogates. At the same time, it uses a compact, interpretable trunk and operates entirely in the reduced space. In particular, Case~II shows that exogenous data (boundary conditions and loads) can be handled efficiently through a small number of offline modes, with online complexity proportional to the reduced dimension.
The present work has several limitations. In Case~II, Dirichlet data are enforced only through a finite boundary mode space, so the boundary conditions are satisfied up to a controlled but nonzero projection error that depends on the mode construction. In addition, the analysis is restricted to linear, coercive elliptic problems with an exact or EIM–based affine operator decomposition and to RB spaces fixed offline. Noncoercive, nonlinear, or strongly nonaffine problems fall outside the current theory.
Future work will address these issues. We plan to develop more precise treatments of boundary constraints and adaptive trunk enrichment driven by \emph{a posteriori} indicators. We also aim to extend the framework to time–dependent, nonaffine, and inverse problems, while retaining quantitative error control and reduced online cost.

\section*{Acknowledgment}
This work was supported by the National Science Foundation (NSF) under grants DMS-2533878, DMS-2053746, DMS-2134209, ECCS-2328241, CBET-2347401 and OAC-2311848, and by the U.S.~Department of Energy (DOE) Office of Science Advanced Scientific Computing Research program under award number DE-SC0023161, and the DOE–Fusion Energy Science program, under grant number: DE-SC0024583.

\bibliographystyle{model1-num-names}
\bibliography{reference}

\end{document}